\documentclass[twoside,11pt]{article}

\usepackage{blindtext}

\usepackage[nohyperref, preprint]{jmlr2e}

\usepackage{url}            
\usepackage{booktabs}       
\usepackage{amsfonts}       
\usepackage{nicefrac}       

\usepackage{amsmath}
\usepackage{enumitem}
\usepackage{wrapfig}
\usepackage{siunitx}

\usepackage{algorithm}
\usepackage[noend]{algpseudocode}
\usepackage{ifthen}
\algnewcommand{\Let}{\textbf{let}\ }
\algnewcommand\algorithmicdef{\textbf{def}}
\algnewcommand\algorithmicby{\textbf{by}}
\algblockdefx[DEF]{Def}{EndDef}[1]
  {\algorithmicdef\ #1\ \algorithmicby}
  {\algorithmicend\ \algorithmicdef}
\makeatletter
\ifthenelse{\equal{\ALG@noend}{t}}%
  {\algtext*{EndDef}}
  {}%
\makeatother

\usepackage{hyperref}

\newcommand{\myref}[2]{\ref{#1}~\ref{#1:#2}}
\newcommand{\lmat}{\begin{bmatrix}}
\newcommand{\rmat}{\end{bmatrix}}

\newcommand{\RR}{\mathbb{R}}
\newcommand{\ZZ}{\mathbb{Z}}
\newcommand{\eps}{\varepsilon}
\newcommand{\crelu}{\operatorname{CReLU}}
\newcommand{\diag}{\operatorname{diag}}
\newcommand{\ran}{\operatorname{ran}}
\newcommand{\relu}{\operatorname{ReLU}}

\newcommand{\st}{\operatorname{st}}
\newcommand{\PZ}{\mathrm{PZ}}

\newcommand{\heavi}{\mathcal{H}}

\ShortHeadings{\texorpdfstring{$G$}{G}-invariant Deep Neural Networks}{Agrawal and Ostrowski}
\firstpageno{1}

\begin{document}

\title{Densely Connected \texorpdfstring{$G$}{G}-invariant Deep Neural Networks with Signed Permutation Representations}

\author{\name Devanshu Agrawal \email dagrawa2@vols.utk.edu \\
       \addr Department of Industrial and Systems Engineering\\
       University of Tennessee\\
       Knoxville, TN 37996, USA
       \AND
       \name James Ostrowski \email jostrows@utk.edu \\
       \addr Department of Industrial and Systems Engineering\\
       University of Tennessee\\
       Knoxville, TN 37996, USA}

\editor{My editor}

\maketitle

\begin{abstract}
We introduce and investigate, for finite groups $G$, $G$-invariant deep neural network ($G$-DNN) architectures with ReLU activation that are densely connected-- 
i.e., include all possible skip connections. 
In contrast to other $G$-invariant architectures in the literature, the preactivations of the$G$-DNNs presented here are able to transform by \emph{signed} permutation representations (signed perm-reps) of $G$. 
Moreover, the individual layers of the $G$-DNNs are not required to be $G$-equivariant; 
instead, the preactivations are constrained to be $G$-equivariant functions of the network input in a way that couples weights across all layers. 
The result is a richer family of $G$-invariant architectures never seen previously. 
We derive an efficient implementation of $G$-DNNs after a reparameterization of weights, 
as well as necessary and sufficient conditions for an architecture to be ``admissible''-- 
i.e., nondegenerate and inequivalent to smaller architectures. 
We include code that allows a user to build a $G$-DNN interactively layer-by-layer, 
with the final architecture guaranteed to be admissible. 
We show that there are far more admissible $G$-DNN architectures than those accessible with the ``concatenated ReLU'' activation function from the literature. 
Finally, we apply $G$-DNNs to two example problems---%
(1) multiplication in $\{-1, 1\}$ (with theoretical guarantees) and (2) 3D object classification---%
finding that the inclusion of signed perm-reps significantly boosts predictive performance compared to baselines with only ordinary (i.e., unsigned) perm-reps.
\end{abstract}


\section{Introduction}
\label{sec:intro}

When fitting a deep neural network (DNN) to a target function that is known to be $G$-invariant with respect to a group $G$, it only makes sense to enforce $G$-invariance on the DNN as prior knowledge. 
With the rise of geometric deep learning~\citep{bronstein2021geometric}, this is becoming an increasingly common practice, 
finding applications in various domains including computer vision~\citep{veeling2018rotation, esteves2018learning, musallam2022leveraging}, the physical sciences~\citep{luo2021gauge, atz2021geometric, kaba2022equivariant, agrawal2023group}, and genomics~\citep{mallet2021reverse, zhou2022towards}. 
In general-purpose $G$-invariant and $G$-equivariant architectures such as $G$-equivariant convolutional neural networks ($G$-CNNs)~\citep{cohen2016group} and $G$-equivariant graph neural networks~\citep{maron2019invariant}, 
it is standard to require every linear layer to be $G$-equivariant. 
Moreover, in case of the rectified linear unit (ReLU) activation, every linear layer is $G$-equivariant only with respect to permutation representations. 
It is commonly assumed that the $G$-invariant architectures constructed in this way are sufficient for consideration~\citep{cohen2019general, finzi2021practical}, 
but it is unclear if this layerwise construction covers all possible ways of enforcing $G$-invariance on a fully-connected feedforward DNN, 
and it remains an open conjecture~\citep{kondor2018generalization}. 
While these architectures and others are certainly sufficient for universal approximation of $G$-invariant functions~\citep{maron2019universality, ravanbakhsh2020universal, kicki2020computationally}, 
the way in which $G$-invariance is enforced is an aspect of the neural architecture and thus likely plays a key role in determining the inductive bias of the model and hence its generalization power on a given problem.

It has recently been discovered that there are, in fact, more ways to enforce $G$-invariance on shallow ReLU neural networks than just permutations on the hidden neurons~\citep{agrawal2022classification}. 
In their work, \citet{agrawal2022classification} exploit the identity
\begin{equation} \label{eq:relu}
\relu(-x) = \relu(x) - x
\end{equation}
to show that $G$-invariance can be achieved even if $G$ acts on the hidden neurons via a ``signed permutation representation'' (signed perm-rep), 
resulting in novel $G$-invariant shallow architectures previously unknown. 
In an attempt towards a generalization to deep architectures, we observe that the linear term in Eq.~\eqref{eq:relu} can be interpreted as a skip connection; 
this suggests that skip connections may be the key to novel deep $G$-invariant architectures.

In this paper, as a partial generalization of the work of \citet{agrawal2022classification}, we investigate $G$-invariant deep neural network ($G$-DNN) architectures that are ``densely connected''-- i.e., include all possible skip connections. 
We note that (non-$G$-invariant) densely connected neural networks exist in the literature and have found immense success especially in medical imaging~\citep{huang2017densely}. 
We use ReLU activation, and every preactivation layer is still a $G$-equivariant function of the network input; 
however, in contrast to previous architectures such as the $G$-CNN, the individual weight matrices of a $G$-DNN need not be $G$-equivariant. 
Instead, in each layer, only the concatenation of the weight matrix with all skip connections from previous layers need be $G$-equivariant. 
This dense structure allows us to use \emph{signed} perm-reps to enforce$G$-equivariance on the preactivation layers, 
thus granting us access to a much larger family of $G$-invariant architectures than seen previously. 

Implementation of $G$-DNNs is nontrivial, 
as the group representation by which the weight matrix (concatenated with all skip connections) in each layer transforms is itself a function of the weights in previous layers. 
That is, due to the skip connections and in contrast to previous architectures such as the $G$-CNN, $G$-invariance is enforced in a way that couples weights across layers. 
We show, however, that there is a reparameterization of $G$-DNNs in which the equivariance conditions of the weight matrices decouple and admit a simple implementation (Thm.~\myref{thm:V}{a}). 
For efficiency, rather than transforming back to the original weights, we express and implement the forward pass of a $G$-DNN directly in terms of these reparameterized weights (Thm.~\myref{thm:V}{b}).

There is additional literature suggesting the potential of $G$-DNNs. 
Signed perm-reps have been used previously as design components in $G$-invariant architectures and have proved beneficial especially in computer vision~\citep{cohen2017steerable}. 
These previous works incorporate signed perm-reps by replacing ReLU with the \emph{concatenated ReLU} (CRELU) activation function~\citep{shang2016understanding} defined as
\begin{equation} \label{eq:crelu}
\crelu(x) = \relu\left( \lmat x \\ -x \rmat \right),
\end{equation}
which sends negations of its input to transpositions of its output. 
CReLU was originally introduced to capture the empirical observation that certain feature maps in convolutional neural networks trained on images tend to pair up~\citep{shang2016understanding}. 
Since then, besides its use in $G$-invariant architectures, CReLU has been shown to be beneficial for the trainability of networks with skip connections-- 
e.g., its compatibility with batch normalization~\citep{shang2017exploring} and its role in the ``looks linear'' weight initialization~\citep{balduzzi2017shattered}. 
The CReLU literature thus suggests that $G$-invariance, signed perm-reps, and skip connections may play well together. 
and we claim that the $G$-DNN architectures introduced in this paper are a culmination of these ideas. 
In contrast to the literature, however, our perspective is top-down, 
and we show in this paper that a systematic study of $G$-DNN architectures leads to the discovery of novel architecture designs not accessible with CReLU alone (Ex.~\ref{ex:crelu}).

Finally, this work contributes to the ultimate goal of $G$-invariant neural architecture search ($G$-NAS), 
which purports to apply search methods over $G$-invariant architectures. 
In addition to finding good $G$-invariant architectures for a given problem, $G$-NAS would relieve practitioners from having to learn or perform specialized mathematics (in particular, group theory), thereby making $G$-invariant deep learning more accessible and applicable. 
Previous works have already begun to explore $G$-NAS with promising results~\citep{basu2021autoequivariant, maile2022equivariance}, 
but these works only operate within a limited region of $G$-invariant architecture space; 
indeed, \citet{agrawal2022classification} argue that a more extensive $G$-NAS first requires the characterization of all $G$-invariant architectures as well as the so-called network morphisms between them. 
There has been work characterizing $G$-invariant architectures from a graph-theoretic perspective but not exploiting identities of the activation function~\citep{ravanbakhsh2017equivariance}; 
work on the classification of $G$-invariant shallow ReLU architectures~\citep{agrawal2022classification}; 
and even work introducing novel $G$-invariant architectures based on a sum-product layer~\citep{kicki2021new}. 
Now, our discovery of novel $G$-DNN architectures based on Eq.~\eqref{eq:relu} pushes the horizon on $G$-invariant architecture space further back, revealing new regions of exploration. 
We prove Thm.~\ref{thm:phi}, which lets us count numbers of distinct non-degenerate $G$-DNNs, in terms of which we gain intuition about how many new architectures we are now able to access.

The remainder of the paper is organized as follows: 
In Sec.~\ref{sec:rho}, we review signed perm-reps and state our central hypothesis with some theoretical support. 
In Sec.~\ref{sec:gdnn}, we introduce and describe the implementation of $G$-DNN architectures. 
We additionally derive necessary and sufficient conditions for a $G$-DNN architecture to be ``admissible'' or nondegenerate (Thm.~\ref{thm:phi}), 
in the sense that (1) no neuron is missing an input and (2) no two ReLU neurons can be combined into a single ReLU neuron or a skip connection. 
We include code%
\footnote{Code for our implementation and for reproducing all results in this paper is available at: \url{https://github.com/dagrawa2/gdnn_code}.} %
that allows a user to build a $G$-DNN interactively such that the final architecture is guaranteed to be admissible. 
We also verify that batch normalization placed after ReLU is compatible with $G$-DNNs out-of-the-box. 
In Sec.~\ref{sec:examples}, we demonstrate that $G$-DNNs go well-beyond CReLU-based architectures, 
and we test $G$-DNNs on two examples---(1) a simple mathematical function and (2) a real-world computer vision problem---and demonstrate that signed perm-reps can, in fact, carry useful inductive bias. 
Finally, in Sec.~\ref{sec:conclusion}, we end with conclusions, implications, and future outlook.

\section{Signed permutation representations}
\label{sec:rho}

\subsection{Preliminaries}

We begin by introducing some notions and notation that will be used throughout this paper. 
This section is an abbreviation of Secs.~2.1-2.2 of \citet{agrawal2022classification}, 
and we refer readers to all of Sec.~2 of that paper for details of the below material.

Throughout this paper, let $G$ be a finite group of $m\times m$ orthogonal matrices. 
Let $\operatorname{P}(n)$ be the group of $n\times n$ permutation matrices 
and $\operatorname{Z}(n)$ the group of $n\times n$ diagonal matrices with diagonal entries $\pm 1$. 
Let $\PZ(n)$ be the group of \emph{signed permutations}-- i.e., 
the group of all permutations and reflections of the standard orthonormal basis $\{e_1,\ldots,e_n\}$. 
For interested readers, this group is the semidirect product $\operatorname{P}(n)\ltimes\operatorname{Z}(n)$, 
and it is also called the hyperoctahedral group in the literature~\citep{baake1984structure}.

A \emph{signed permutation representation} (signed perm-rep) of degree $n$ of $G$ is a homomorphism $\rho:G\mapsto\PZ(n)$. 
Two signed perm-reps $\rho,\rho^{\prime}$ are said to be \emph{equivalent} or \emph{conjugate} if there exists $A\in\PZ(n)$ such that $\rho^{\prime}(g) = A\rho(g)A^{-1}\forall g\in G$-- 
i.e., if they are related by a change of basis. 
A signed perm-rep $\rho$ is said to be \emph{reducible} if it is equivalent to a direct sum of signed perm-reps of smaller degrees-- 
i.e., if there exists $A\in\PZ(n)$ such that $A\rho(\cdot)A^{-1}$ is simultaneously block-diagonal with at least two blocks. 
The signed perm-rep is said to be \emph{irreducible} (signed perm-irrep for short) otherwise.%
\footnote{%
A useful characterization is that a signed perm-rep $\rho$ is irreducible iff for every $i,j=1,\ldots,n$, there exists $g\in G$ such that $\rho(g)e_i = \pm e_j$.%
}

The signed perm-irreps of $G$ can be completely classified up to equivalence in terms of certain pairs of subgroups of $G$~\citep[][Thm.~1]{agrawal2022classification}. 
For every pair of subgroups $K\leq H\leq G$, $|H:K|\leq 2$, define the signed perm-rep $\rho_{HK}:G\mapsto\PZ(n)$, $n = |G/H|$, to be the induced representation
\[ \rho_{HK} = I_H^G \sigma, \mbox{ where } \sigma:H\mapsto\{-1, 1\}\mid \ker(\sigma) = K. \]
Then $\rho_{HK}$ is irreducible, and every signed perm-irrep is equivalent to some $\rho_{HK}$. 
Moreover, $\rho_{HK}$ and $\rho_{H^{\prime}K^{\prime}}$ are equivalent iff $(H^{\prime}, K^{\prime}) = (gHg^{-1}, gKg^{-1})$ for some $g\in G$. 
We can thus always understand a ``signed perm-irrep'' to mean $\rho_{HK}$ for some appropriate subgroups $H$ and $K$, 
and conversely we will always understand the notation $\rho_{HK}$ to mean the above construction.

For a more explicit but equivalent construction of $\rho_{HK}$, let $\{g_1,\ldots,g_n\}$ be a transversal of $G/H$. 
If $|H:K| = 1$ (i.e., $H=K$), then we define $\rho_{HK}(g)e_i = e_j$ iff $gg_i H = g_j H$; 
i.e., $\rho_{HK}$ is just defined in terms of its permutation action on $G/H$. 
On the other hand, suppose $|H:K|=2$. Then we have the quotient group $H/K = \{K, hK\}$. 
Define $e_{-i} = -e_i$ and $g_{-i} = g_i h$ for $i\in\{1,\ldots,n\}$. 
Then we define $\rho_{HK}(g)e_i = e_j$ iff $gg_i K = g_j K$ for $i,j\in\{\pm 1,\ldots,\pm n\}$. 
Intuitively, $\rho_{HK}$ still acts on $G/H$ via permutations; 
however, if the cosets in $H/K$ are swapped under this action, then a sign flip is incurred in the representation. 
For more details on these constructions, see Sec.~2 of \citet{agrawal2022classification}. 

Every signed perm-irrep is either type 1 or type 2. 
A signed perm-rep is \emph{type 1} if it is equivalent to an \emph{ordinary permutation representation} (ordinary perm-rep) $\pi:G\mapsto\operatorname{P}(n)$; 
it is \emph{type 2} otherwise. 
Sign flips in a type 1 signed perm-rep are thus artifacts as they can be removed by a change of basis. 
An important characterization is that a signed perm-irrep $\rho_{HK}$ has type $|H:K|$.

Finally, as general notation, throughout this paper let $I_n$ denote the $n\times n$ identity matrix and $\vec{0}_n$ (resp. $\vec{1}_n$) the $n$-dimensional vector with all elements $0$ (resp. $1$).

\subsection{Central hypothesis}
\label{sec:central}

Previous works in the equivariant deep learning literature, to our knowledge, primarily employ type 1 signed perm-reps---and in particular, ordinary perm-reps---to enforce layerwise $G$-equivariance in ReLU networks. 
The goal of this paper is to show that type 2 signed perm-reps can also be implemented for this purpose 
(although it is trickier as it requires layers to be coupled), 
and the central hypothesis of this work is that type 2 signed perm-reps can indeed be beneficial for DNN performance. 
In this section, we motivate this hypothesis with theory and discussions that suggest that type 2 signed perm-reps can help boost expressive power without necessarily sacrificing generalization power. 
We note that the CReLU activation function has previously been used to incorporate certain type 2 signed perm-reps (see Sec.~\ref{sec:intro}), but we take a top-down approach that lets us study all such architectures.

Our first theorem states that two inequivalent signed perm-irreps induced from different reps of the same subgroup $H\leq G$ correspond to $G$-equivariant matrices that are in a sense orthogonal.%
\footnote{All proofs can be found in the supplementary material.}

\begin{theorem}[Corollary to Prop.~6 of \citet{agrawal2022classification}] \label{thm:orthogonal}
Let $\rho_{HK_1}$ and $\rho_{HK_2}$ be inequivalent signed perm-irreps of degree $n$ of $G$. 
For each $i\in\{1,2\}$, let $W_i\in\RR^{n\times m}$ such that $\rho_{HK_i}(g)W_i = W_i g\forall g\in G$. 
Then there exists $P\in\operatorname{P}(n)$ such that
\[ \diag(PW_1W_2^{\top}) = 0. \]
\end{theorem}

Theorem~\ref{thm:orthogonal} will be important in Sec.~\ref{sec:examples}, where we will use it to construct a key baseline $G$-DNN architecture for our experiments. 
We discuss Thm.~\ref{thm:orthogonal} further in the context of an example group $G$ later in this section.

Our second theorem describes how a signed perm-rep can be ``unraveled'' into an ordinary perm-rep of twice the degree, 
by essentially mapping sign flips to transpositions of two dimensions. 
This unraveling is particularly meaningful for type 2 signed perm-irreps, which remain irreducible even after unraveling; 
in contrast, type 1 signed perm-irreps merely unravel into reducible copies of themselves. 
Let $\heavi$ be the elementwise Heaviside step function, where we let $\heavi(0) = 0$.

\begin{theorem} \label{thm:signed2ordinary}
Let $\rho:G\mapsto\PZ(n)$ be a signed perm-rep. 
Then:
\begin{enumerate}[label={(\alph*)}]
\item \label{thm:signed2ordinary:a}
The function%
\footnote{In the expression for $\pi_{\rho}$, the operation $\otimes$ is the Kronecker product.}
\[ \pi_{\rho}(g) = \heavi\left(\lmat 1 & -1\\ -1 & 1\rmat \otimes \rho(g) \right) \]
defines an ordinary perm-rep $\pi_{\rho}:G\mapsto\operatorname{P}(2n)$.
\item \label{thm:signed2ordinary:b}
Let $W\in\RR^{n\times m}$ such that $\rho(g)W = Wg\forall g\in G$. 
Then
\[ \pi_{\rho}(g)\lmat W\\ -W\rmat = \lmat W\\ -W\rmat g\forall g\in G. \]
\item \label{thm:signed2ordinary:c}
Suppose $\rho$ is irreducible. 
Then $\pi_{\rho}$ is irreducible iff $\rho$ is type 2.
\item \label{thm:signed2ordinary:d}
Suppose $\rho$ is type 2 irreducible and $\rho = \rho_{HK}$. 
Then $\pi_{\rho} = \rho_{KK}$.
\end{enumerate}
\end{theorem}

\begin{figure}
\centering
\begin{minipage}{\textwidth}
\begin{minipage}{0.3\textwidth}
$W_{(6, 1)}$

\includegraphics[trim={0 0 0 0}, clip, width=0.9\textwidth]{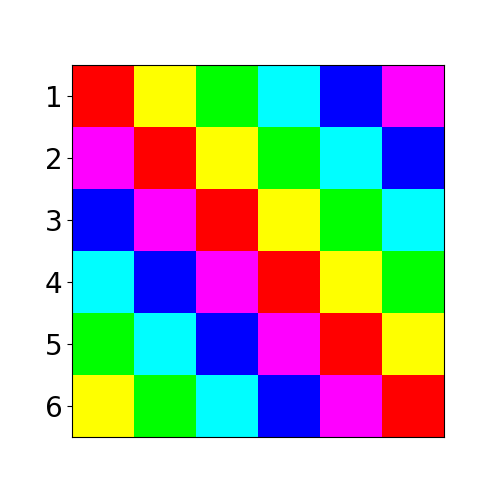}
\end{minipage}%
\begin{minipage}{0.3\textwidth}
$W_{(3, 1)}$ 

\includegraphics[trim={0 0 0 0}, clip, width=0.9\textwidth]{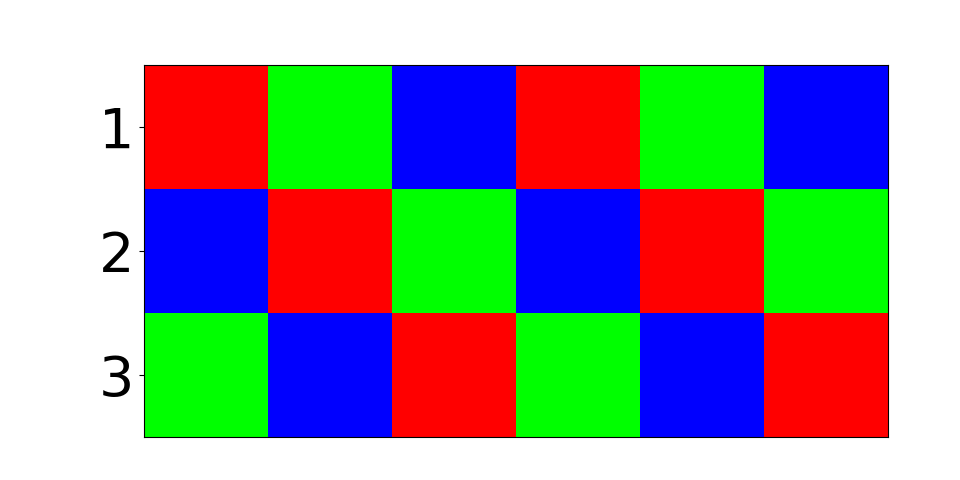}
\end{minipage}%
\begin{minipage}{0.3\textwidth}
$W_{(3, 2)}$

\includegraphics[trim={0 0 0 0}, clip, width=0.9\textwidth]{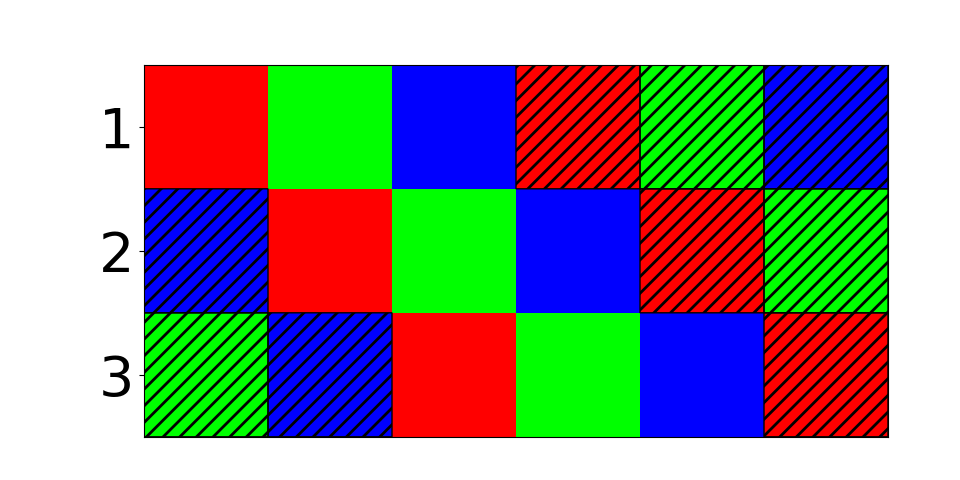}
\end{minipage}
\end{minipage}
\caption{\label{fig:patterns} %
Weightsharing patterns of equivariant matrices for three example signed perm-irreps of the group $G$ of $6\times 6$ cyclic permutation matrices 
(see main text for details). 
In each pattern, weights of the same color and texture (solid vs. hatched) are constrained to be equal; 
weights of the same color but different texture are constrained to be opposites 
(colors should not be compared across different matrices).
}
\end{figure}

To understand Thms.~\ref{thm:orthogonal}-\ref{thm:signed2ordinary} and their implications for the expressive and generalization powers of a $G$-DNN with type 2 signed perm-reps, 
consider the example group $G$ of all $6\times 6$ cyclic permutation matrices 
(see Sec.~4.1 of \citet{agrawal2022classification} for a detailed discussion of this example). 
For this example, the signed perm-irreps of $G$ are completely characterized in terms of their degrees and types; 
we thus let $\rho_{(n, t)}$ denote the irrep $\rho_{HK}$ where $n = \frac{6}{|H|}$ is the degree and $t=|H:K|$ is the type. 
Alternatively, if $G\cong\ZZ_6$, then $\rho_{(n, t)}$ denotes the irrep $\rho_{HK}$ where $H\cong\ZZ_{\frac{6}{n}}$ and $K\cong\ZZ_{\frac{6}{nt}}$. 
In this notation, the signed perm-irreps of $G$ are $\rho_{(6,1)}$, $\rho_{(3,1)}$, $\rho_{(3,2)}$, $\rho_{(2,1)}$, $\rho_{(1,1)}$, and $\rho_{(1,2)}$.

For each irrep $\rho_{(n, t)}$, let $W_{(n, t)}\in\RR^{n\times m}$ such that $\rho_{(n,t)}W_{(n,t)} = W_{(n,t)}g\forall g\in G$. 
The irreps $\rho_{(3,1)}$ and $\rho_{(3,2)}$ satisfy the hypotheses of Thm.~\ref{thm:orthogonal}, 
and indeed we see---in terms of their weightsharing patterns---each row of $W_{(3,1)}$ is orthogonal to the corresponding row of $W_{(3,2)}$ (Fig.~\ref{fig:patterns}). 
In practice, this means if we were given a dataset sampled from a $G$-invariant shallow neural network ($G$-SNN)
\[ f(x) = a\vec{1}^{\top}\relu(Wx+b\vec{1}) + c^{\top}x + d, \]
with $W=W_{(3,2)}$ as the ground truth weight matrix, 
then a $G$-SNN with weight matrix  with the weightsharing pattern of $W_{(3,1)}$ would not have the expressive power to fit the ground truth, 
regardless of the number of ``channels'' (i.e., independent copies of $W_{(3,1)}$) used. 
With a limited budget of three hidden neurons per channel, type 2 signed perm-reps thus help to increase expressive power.

On the other hand, if we double our budget of hidden neurons, then Thm.~\ref{thm:signed2ordinary} suggests $W_{(6,1)}$ has the capacity to express $W_{(3,2)}$; 
indeed, if we constrain the first row of $W_{(6,1)}$ to have the same weightsharing pattern as the first row of $W_{(3,2)}$ (Fig.~\ref{fig:patterns}), 
then $W_{(6,1)}$ effectively becomes equivalent to $W_{(3,2)}$. 
The problem is, however, that $W_{(6,1)}$ can similarly be constrained to match $W_{(3,1)}$, 
and with probability $1$ (e.g., under a Gaussian), $W_{(6,1)}$ is equivalent to neither $W_{(3,1)}$ nor $W_{(3,2)}$. 
Thus, while this approach allows us to use a type 1 signed perm-rep with the capacity to express the ground truth, 
it comes at the cost of generalization power, 
as $W_{(6,1)}$ may have multiple configurations consistent with a finite training set.

To ensure these concepts are made clear, we present one more very simple example. 
For any even input dimension $n\geq 2$, consider the $G$-SNN architectures
\[ f_i(x) = \vec{1}^{\top}\relu(W_i x) - \frac{1}{2}\vec{1}^{\top}W_i x\forall i\in\{1,2,3\}, \]
where
\begin{align*}
W_1 &= \lmat u & u & \cdots & u & u \rmat \\
W_2 &= \lmat u & -u & \cdots & u & -u \rmat \\
W_3 &= \lmat u & v & \cdots & u & v \rmat,
\end{align*}
and where $\vec{1}$ is a vector of $1$'s of the appropriate length (either 1 or 2). 
Here $G$ is the group of cyclic permutations on $n$ elements. 
To fix a global scale, without loss of generality, let us set $u=1$. 
Now suppose we take the type 2 architecture $f_2$ as the ground truth, and draw a sample dataset. 
We then observe that the type 1 architecture $f_1$ does not have the capacity to fit this dataset; 
in fact, $f_1$ and $f_2$ have orthogonal level sets. 
On the other hand, the unraveled architecture $f_3$ does have the capacity to fit the data, 
but it is a weaker model compared to $f_2$ (with respect to this dataset) as it must learn $v=-u$ from the data. 

In sum, type 2 signed perm-reps may help to refine the range of expressive powers available to us; 
they help to express functions different from type 1 signed perm-reps of the same degree 
without sacrificing generalization power the way larger type 1 signed perm-reps would.

\section{\texorpdfstring{$G$}{G}-invariant deep neural networks}
\label{sec:gdnn}

\subsection{Parameterization redundancies}
\label{sec:parameter}

In this section, we define densely connected deep neural networks, or simply deep neural networks, 
and we list some of the so-called parameterization redundancies of such networks that will be key to enforcing $G$-invariance. 
All notation introduced here will persist throughout the paper.

Let $f:\RR^m\mapsto\RR$ be a \emph{deep neural network} (DNN) of depth $d$ constructed as follows: 
Let $\{n_1,\ldots,n_{d+1}\}$ be a set of $d+1$ positive integers where%
\footnote{Note the input dimension $n_1=m$ is the same number as the degree of the matrix group $G$ in Sec.~\ref{sec:rho}.} %
$n_1=m$ and $n_{d+1}=1$, 
and let $N_i=n_1+\cdots +n_i$ for every $i\in\{1,\ldots,d\}$. 
Let $W^{(i)}\in\RR^{n_{i+1}\times N_i}$ and $b^{(i)}\in\RR^{n_{i+1}}$ for every $i\in\{1,\ldots,d\}$. 
Define $f^{(1)}:\RR^m\mapsto \RR^{n_1}$ by $f^{(1)}(x) = x$. 
For every $i\in\{1,\ldots,d-1\}$, define $f^{(i+1)}:\RR^m\mapsto\RR^{N_{i+1}}$ by
\begin{equation} \label{eq:fi}
f^{(i+1)}(x) = \lmat
\relu(W^{(i)} f^{(i)}(x) + b^{(i)}) \\
f^{(i)}(x)
\rmat.
\end{equation}
Then, let $f(x) = W^{(d)}f^{(d)}(x) + b^{(d)}$.

We thus define a DNN to be a feedforward ReLU network having all possible \emph{skip connections}. 
We call the $W^{(i)}$ \emph{weight matrices}, their rows \emph{weight vectors}, and the $b^{(i)}$ \emph{bias vectors}. 
The numbers $n_2, \ldots, n_d$ are the widths of the hidden layers, and the numbers $N_2, \ldots, N_d$ are the cumulative widths accounting for the skip connections. 
Note the traditional definition of a DNN can be recovered by setting all skip connections to zero. 

Our first proposition below establishes a set of parameterization redundancies (i.e., reparameterizations of the DNN leaving the input-output function invariant) enjoyed by the above DNN architecture. 
For every positive integer $n$, let $\operatorname{C}(n)$ be the group of $n\times n$ diagonal matrices with positive diagonal entries.

\begin{proposition} \label{prop:CPZ}
Let $i\in\{1,\ldots,d-1\}$, 
$C\in \operatorname{C}(n_{i+1})$, 
$P\in\operatorname{P}(n_{i+1})$,  and
$Z\in\operatorname{Z}(n_{i+1})$. 
Then the DNN $f$ is invariant under the transformation
\begin{align*}
W^{(i)} &\rightarrow CPZ W^{(i)} \\
b^{(i)} &\rightarrow CPZ b^{(i)} \\
W^{(i+1)} &\rightarrow W^{(i+1)}
\lmat
(CP)^{-1} & \heavi(-Z)W^{(i)} \\
0 & I_{N_i}
\rmat \\
b^{(i+1)} &\rightarrow b^{(i+1)} + W^{(i+1)}
\lmat
\heavi(-Z) b^{(i)} \\
0
\rmat.
\end{align*}
\end{proposition}

Proposition~\ref{prop:CPZ} is the key to understanding how signed perm-reps can be used in $G$-DNNs. 
Specificly, if we want the parameters $(W^{(i)}, b^{(i)})$ of the $i$th layer to transform by a signed perm-rep, then the parameters $(W^{(i+1)}, b^{(i+1)})$ of the subsequent layer must transform as in Prop.~\ref{prop:CPZ} to compensate and maintain invariance. 
This idea is the basis of the next section.

\subsection{\texorpdfstring{$G$}{G}-invariant architectures}
\label{sec:equivariant}

The following lemma gives a sufficient condition for each $f^{(i)}$ (the subnetwork comprising the first $i-1$ layers of $f$ and acting as input of the $i$th layer) to be $G$-equivariant; 
the sufficient condition takes the form of equivariant constraints on the network parameters. 

\begin{lemma} \label{lemma:psi}
Let $\{\rho^{(1)},\ldots,\rho^{(d)}\}$ be a set of signed perm-reps $\rho^{(i)}:G\mapsto\PZ(n_{i+1})$ with $\rho^{(d)}$ the trivial rep. 
For each $i\in\{1,\ldots,d\}$, let $\pi^{(i)}:G\mapsto\operatorname{P}(n_{i+1})$ and $\zeta^{(i)}:G\mapsto\operatorname{Z}(n_{i+1})$ be the unique functions such that $\rho^{(i)}(g) = \pi^{(i)}(g)\zeta^{(i)}(g)\forall g\in G$. 
Let $\{\psi^{(1)},\ldots,\psi^{(d)}\}$ be a set of reps $\psi^{(i)}:G\mapsto\operatorname{GL}(N_i)$ defined as:
\begin{align*}
\psi^{(1)}(g) &= g \\
\psi^{(i+1)}(g) &= 
\lmat
\pi^{(i)}(g) & \frac{1}{2}(W^{(i)}\psi^{(i)}(g) - \pi^{(i)}(g)W^{(i)}) \\
0 & \psi^{(i)}(g)
\rmat
\forall i\in\{1,\ldots,d-1\}.
\end{align*}
Suppose $\rho^{(i)}(g)W^{(i)} = W^{(i)}\psi^{(i)}(g)$ and $\rho^{(i)}(g)b^{(i)} = b^{(i)}$ for all $g\in G$ and $i\in\{1,\ldots,d\}$. 
Then $f^{(i)}(gx) = \psi^{(i)}(g) f^{(i)}(x)$ for all $g\in G$, $x\in\RR^m$, and $i\in\{1,\ldots,d\}$.
\end{lemma}

Since $f^{(d)}(gx) = \psi^{(d)}f^{(d)}(x)$ and $W^{(d)}\psi^{(d)}(g) = W^{(d)}$ for all $g\in G$ and $x\in\RR^m$, then we see that Lemma~\ref{lemma:psi} gives a sufficient condition for $f$ to be a \emph{$G$-invariant deep neural network} ($G$-DNN). 
Note the sequence of signed perm-reps $\{\rho^{(1)},\ldots,\rho^{(d)}\}$ completely determines the $G$-DNN architecture (i.e., depth, number of hidden neurons per layer, and weightsharing patterns), 
and hence we will also refer to such sequences of reps as \emph{$G$-DNN architectures}. 
Observe that the rep $\psi^{(i)}$ depends on the weight matrix $W^{(i-1)}$, 
and hence the equivariance condition on $W^{(i)}$ introduces a coupling between $W^{(i-1)}$ and $W^{(i)}$. 
This coupling makes it difficult to implement the equivariance constraints on the weight matrices directly, 
and we thus proceed to find a reparameterization admitting uncoupled equivariant weight matrices. 
To do this, we first state another lemma below, 
which gives an explicit formula for the reps $\psi^{(i)}$ (as opposed to a recursive formula) and shows each $\psi^{(i)}$ to be equivalent to a direct sum of layerwise reps.

For every $i\in\{1,\ldots,d\}$, decompose $W^{(i)}$ into blocks as
\[ W^{(i)} = \lmat W^{(i)}_i & W^{(i)}_{i-1} & \cdots & W^{(i)}_1 \rmat, \]
where $W^{(i)}_j\in\RR^{n_{i+1}\times n_j}$. 
Define the block matrix
\[ A^{(i)} = 
\lmat
I_{n_{i+1}} & -\frac{1}{2}W^{(i)}_i & -\frac{1}{2}W^{(i)}_{i-1} & \cdots & -\frac{1}{2}W^{(i)}_1 \\
0 & I_{n_i} & -\frac{1}{2}W^{(i-1)}_{i-1} & \cdots & -\frac{1}{2}W^{(i-1)}_1 \\
0 & 0 & I_{n_{i-1}} & \cdots & -\frac{1}{2}W^{(i-2)}_1 \\
0 & 0 & 0 & \ddots & \vdots \\
0 & 0 & 0 & 0 & I_{n_1}
\rmat. \]
Define the rep $\Pi^{(i)}:G\mapsto\operatorname{O}(N_{i+1})$ by
\[ \Pi^{(i)}(g) = \diag(\pi^{(i)}(g), \pi^{(i-1)}(g), \ldots, \pi^{(1)}(g), g). \]

\begin{lemma} \label{lemma:unfold}
For every $i\in\{1,\ldots,d-1\}$, we have%
\footnote{%
The notation $A^{(i)-1}$ is shorthand for $(A^{(i)})^{-1}$, the matrix inverse of $A^{(i)}$.%
}
\[ \psi^{(i+1)}(g) = A^{(i)-1} \Pi^{(i)}(g) A^{(i)}\forall g\in G. \]
\end{lemma}

Let $\pi^{(0)}:G\mapsto\operatorname{P}(n_1)$ be the identity rep $\pi^{(0)}(g) = g$, 
and let $A^{(0)} = I_{n_1}$. 
Then Lemma~\ref{lemma:unfold} holds for $i=0$ as well.

We are now ready to state this section's key theorem, 
which gives a reparameterization of a $G$-DNN in which it admits uncoupled equivariant weight matrices.

\begin{theorem} \label{thm:V}
Let $f$ be $G$-invariant, with its parameters $\{(W^{(1)}, b^{(1)}), \ldots, (W^{(d)}, b^{(d)})\}$ satisfying the conditions of Lemma~\ref{lemma:psi} with respect to the signed perm-reps $\{\rho^{(1)},\ldots,\rho^{(d)}\}$. 
\begin{enumerate}[label={(\alph*)}]
\item \label{thm:V:a}
For every $i\in\{1,\ldots,d\}$, there exists $V^{(i)}\in\RR^{n_{i+1}\times N_i}$ with block structure
\[ V^{(i)} = \lmat V^{(i)}_i & V^{(i)}_{i-1} & \cdots & V^{(i)}_1 \rmat \]
where $V^{(i)}_j\in\RR^{n_{i+1}\times n_j}$ such that
\begin{align*}
\rho^{(i)}(g)V^{(i)}_j &= V^{(i)}_j \pi^{(j-1)}(g)\forall g\in G \\
W^{(i)} &= V^{(i)}A^{(i-1)}.
\end{align*}
\item \label{thm:V:b}
For every $i\in\{1,\ldots,d\}$, define $g^{(i)}:\RR^m\mapsto\RR^{n_{i+1}}$ by $g^{(i)}(x) = W^{(i)}f^{(i)}(x)$ 
and $h^{(i)}:\RR^m\mapsto\RR^{N_i}$ by $h^{(i)}(x) = A^{(i-1)}f^{(i)}(x)$. 
Then we have the recursion
\begin{align*}
g^{(i)}(x) &= V^{(i)} h^{(i)}(x)\forall i\in\{1,\ldots,d\} \\
h^{(i+1)}(x) &= 
\lmat
\relu(g^{(i)}(x)+b^{(i)}) - \frac{1}{2}g^{(i)}(x) \\
h^{(i)}(x)
\rmat\forall i\in\{1,\ldots,d-1\}.
\end{align*}
\end{enumerate}
\end{theorem}

The $V^{(i)}$ are the latent weight matrices, 
and in Thm.~\myref{thm:V}{a} we see that each $V^{(i)}$ satisfies a linear equivariant condition that is easy to implement. 
The apparent weight matrices $W^{(i)}$ can then be reconstructed as $W^{(i)} = V^{(i)}A^{(i-1)}$. 
In practice, however, there is no need to perform this reconstruction; 
instead, we implement the recursion in Thm.~\myref{thm:V}{b} as the forward pass of the $G$-DNN, 
where the final network output is $f(x) = g^{(d)}(x) + b^{(d)}$. 
Observe that the recursion is given directly in terms of the $V^{(i)}$ as opposed to the $W^{(i)}$. 
This recursion leverages the block-triangular structure of $A^{(i-1)}$ in the transformation $W^{(i)} = V^{(i)}A^{(i-1)}$. 
We describe a concrete implementation of the equivariance condition and the forward pass of the $G$-DNN in App.~\ref{app:forward}.

\subsection{Admissible architectures}
\label{sec:admissible}

Theorem~\ref{thm:V} tells us how to construct a $G$-DNN, but it first requires us to select a sequence of reps or ``architecture'' $\rho^{(1)}, \ldots, \rho^{(d)}$. 
Selecting the optimal sequence is the problem of $G$-invariant neural architecture design, which is beyond the scope of this work. 
At the very least, however, we would like to avoid sequences that correspond to ``degenerate'' network architectures. 
In particular, we require that 
(1) every row of each weight matrix $W^{(i)}$ be nonzero and 
(2) no two rows of the augmented weight matrix $[W^{(i)}\mid b^{(i)}]$ be parallel. 
The first condition ensures there are no neurons disconnected from all previous layers, 
and the second condition ensures no two hidden neurons in a given layer can be combined into a single hidden neuron or skip connection. 
We call a sequence of reps $\rho^{(1)}, \ldots, \rho^{(d)}$ an \emph{admissible architecture} if it admits a $G$-DNN with weight matrices satisfying these two conditions. 
Below, Thm.~\ref{thm:phi} provides a characterization of admissible architectures that we implement in practice. 
First, however, we introduce additional notions and notation.

For every $A\in\RR^{m\times m}$, let $\st_G(A)$ be the stabilizer subgroup
\[ \st_G(A) = \{g\in G: gA = A\}, \]
and for every finite orthogonal matrix group $\Gamma$, let $P_{\Gamma}$ be the orthogonal projection operator onto the vector subspace pointwise-invariant under the action of $\Gamma$:
\[ P_{\Gamma} = \frac{1}{|\Gamma|} \sum_{g\in\Gamma} g. \]
Let $\mathcal{S}(G)$ denote the set of all subgroups of $G$. 
Then, we define the function $\theta:\{(H, K, J)\in\mathcal{S}(G)^3: |H:K|\leq 2\}\mapsto\mathcal{S}(G)$ by
\begin{align*}
\theta(H, K, J) 
&= \pi_J^{-1}[\st_{\pi_J(G)}(P_{\pi_J(K)}-(|H:K|-1)P_{\pi_J(H)})] \\
&= \{g\in G: \pi_J(g)(P_{\pi_J(K)}-(|H:K|-1)P_{\pi_J(H)}) = P_{\pi_J(K)}-(|H:K|-1)P_{\pi_J(H)}\},
\end{align*}
where the ordinary perm-rep $\pi_J:G\mapsto\operatorname{P}(|G/J|)$ is defined as usual-- 
i.e., defined to be equivalent to the action of $G$ on $G/J$. 
Note $\pi_J$ is defined only up to conjugation by a perm-matrix, since no ordering on the cosets in $G/J$ is specified. 
It turns out, however, that $\theta$ is invariant under conjugation of $\pi_J$, 
and more generally $\theta$ is invariant and equivariant with respect to certain conjugations of the input subgroups (see Prop.~\ref{prop:theta_conj} in App.~\ref{app:admissible}). 

Let $f:\RR^m\mapsto\RR$ be a $G$-DNN with sequence of signed perm-reps $\rho^{(1)}, \ldots, \rho^{(d)}$. 
For every $i\in\{1,\ldots,d\}$, the signed perm-rep $\rho^{(i)}$ admits the following decomposition into irreducibles:
\[ \rho^{(i)} = \bigoplus_{j=1}^{r^{(i)}} k^{(i)}_j \rho^{(i)}_j, \quad \rho^{(i)}_j = \rho^{(i)}_{H^{(i)}_jK^{(i)}_j}. \]
The irreps $\rho^{(i)}_1, \ldots, \rho^{(i)}_{r^{(i)}}$ are inequivalent, 
and each $k^{(i)}_j$ is a positive integer where we define the notation
\[ k^{(i)}_j \rho^{(i)}_j = \bigoplus_{k=1}^{k^{(i)}_j} \rho^{(i)}_j. \]
We say the $i$th layer of the $G$-DNN $f$ has $r^{(i)}$ distinct \emph{irreps} and $k^{(i)}_1+\cdots +k^{(i)}_{r^{(i)}}$ \emph{channels}. 
Now for every $i\in\{1,\ldots,d\}$, we define the functions $\phi^{(i)}:\{(H, K)\in\mathcal{S}(G)^2: |H:K|\leq 2\}\mapsto\mathcal{S}(G)$ recursively by
\begin{align*}
\phi^{(1)}(H, K) &= \st_G(P_K-(|H:K|-1)P_H) \\
\phi^{(i+1)}(H, K) &= \phi^{(i)}(H, K)\cap \bigcap_{j=1}^{r^{(i)}} \theta(H, K, H^{(i)}_j)\forall i\in\{1,\ldots,d-1\}.
\end{align*}
Following from the equivariance of $\theta$ (Prop.~\ref{prop:theta_conj}) and that inner automorphisms respect subgroup intersection, 
we have the important property that each $\phi^{(i)}$ is equivariant with respect to pairwise conjugation:
\[ \phi^{(i)}(gHg^{-1}, gKg^{-1}) = g\phi^{(i)}(H, K)g^{-1}\forall g\in G. \]

We are ready to state the theorem characterizing single-channel admissible architectures; 
note the following is a generalization of Thm.~4a of \citet{agrawal2022classification}.

\begin{theorem} \label{thm:phi}
Maintain the notation introduced in the last paragraph, 
and suppose $k^{(i)}_j = 1\forall i,j$ (single-channel architecture). 
Then the architecture%
\footnote{By definition of $G$-DNN architecture, $\rho^{(d)}$ is trivial; see Lemma~\ref{lemma:psi} and the discussion immediately proceeding it.} %
$\{\rho^{(1)}, \ldots, \rho^{(d)}\}$ is admissible iff the following conditions hold:
\begin{enumerate}
\item $\phi^{(i)}(H^{(i)}_j, K^{(i)}_j) = K^{(i)}_j\forall i,j$.
\item If $H^{(1)}_j=G$ for some $j$, then $P_G\neq 0$.%
\footnote{This condition is trivially satisfied if $G$ is a permutation matrix group.}
\end{enumerate}
\end{theorem}

Theorem~\ref{thm:phi} gives us a practical way to build admissible $G$-DNN architectures layer-by-layer as follows: 
Suppose we have already built the first $i$ layers-- 
i.e., we have selected the signed perm-reps $\rho^{(1)}, \ldots, \rho^{(i)}$ satisfying Thm.~\ref{thm:phi}. 
Then the function $\phi^{(i+1)}$ is defined. 
We enumerate all signed perm-irreps $\rho_{HK}$, up to equivalence,%
\footnote{The equivariance of $\phi$ is why enumeration of the signed perm-irreps only up to equivalence is sufficient.} %
such that $\phi(H, K) = K$, 
and we then select a subset to comprise $\rho^{(i+1)}$. 
In terms of implementation, the computation of $\phi^{(i)}(H, K)$ boils down to the computation of $\theta(H, K, H^{(i-1)}_j)$ for each $j$, 
which can be accomplished using Alg.~\ref{alg:theta} (see App.~\ref{app:admissible} for details). 
We should think of the outputs $\theta(H, K, J)$ as entries of a precomputed table with a row for each $(H, K)$ (up to pairwise conjugation) and a column for each $J$ (up to conjugation). 
We give additional implementation details in the next section, 
including how we accomodate multiple input and output channels per layer.

To gain some intuition about Thm.~\ref{thm:phi}, we briefly consider its application to $G$-SNN (i.e., depth $d=2$) architectures, 
which are completely characterized in terms of the single signed perm-rep by which the weight matrix in the first layer transforms. 
In this case, Thm.~\ref{thm:phi} reduces to Thm.~4~(a) of \citet{agrawal2022classification}; 
a signed perm-rep $\rho_{HK}$ is admissible for a $G$-SNN iff $\phi^{(1)}(H, K) = K$, 
or more explicitly, 
\[ \st_G(P_K-(|H:K|-1)P_H) = K. \]
Here the weight vectors of the $G$-SNN are $\{gw: gH\in G/H\}$, where $w$ is fixed under $K$ and flips sign under $H\setminus K$. 
The weight vectors are pairwise nonparallel if $\st_G(w) = K$; 
if, however, $\st_G(w) > K$, then the $G$-orbit of $w$ is smaller than the number of weight vectors, forcing some weight vectors to be equal up to sign. 
Theorem~\ref{thm:phi} serves to eliminate such degenerate architectures.

\begin{table}
\centering
\caption{\label{table:admissible} %
Ratio of the number of admissible $G$-DNN architectures to the total number of architectures for every depth and every group $G$, $|G| = 8$, up to isomorphism. 
Only architectures corresponding to sequences of irreps of strictly decreasing degree are considered.
}
\begin{tabular}{cccccc}
\toprule
Depth & $C_8$ & $C_2\times C_4$ & $C_2^3$ & $D_4$ & $Q_8$ \\
\midrule
2 & 5/5 & 8/15 & 11/43 & 14/21 & 9/9 \\
3 & 8/8 & 30/62 & 93/434 & 65/104 & 20/20 \\
4 & 4/4 & 48/48 & 392/392 & 84/84 & 12/12 \\
\bottomrule
\end{tabular}
\end{table}

The number of admissible architectures can be significantly less than the total number of architectures. 
For example, we consider one particular permutation representation of every group of order $8$ up to isomorphism, 
following the constructions described in App.~C.1 of \citet{agrawal2022classification}. 
For each group, we only count the architectures corresponding to sequences of irreps of strictly decreasing degree, 
and we report the fraction of these that are admissible (Table~\ref{table:admissible}). 
Observe that the reduction from the total number of architectures to only the admissible ones is often significant, 
which could be exploited perhaps in a future implementation of $G$-NAS. 
Note also that all architectures of maximum depth are admissible; 
we speculate this is because as depth increases and each new weight matrix $W^{(i)}$ grows in width (due to more skip connections to previous layers), the conditions under which two weight vectors are parallel become harder to satisfy. 
A complete answer, however, is left for future investigation.

\subsection{Additional remarks}
\label{sec:remarks}

\paragraph{Channels} %
The $G$-DNN supports multiple channels in a way generalizing the notion from traditional convolutional neural networks. 
As already mentioned in Sec.~\ref{sec:admissible}, the $i$th layer of the $G$-DNN $f$ is said to have $k^{(i)}_j$ output channels or copies of the irrep $\rho^{(i)}_j$ and $k^{(i-1)}_j$ input channels from the irrep $\rho^{(i-1)}_j$. 
As a simpler case, suppose $\rho^{(i)}$ contains the same number $k^{(i)}$ of copies of each of its constituant irreps, 
and suppose the input $x$ has $k^{(0)}$ channels. 
Then the $i$th layer can be said to have $k^{(i)}$ output channels and $k^{(i-1)}$ input channels. 
In practice, we implement the $G$-DNN assuming only one copy of each irrep $\rho^{(i)}_j$, 
and then we regard each element of the input $x$ as a $k^{(0)}$-dimensional vector 
and each element of the latent weight matrix block $V^{(i)}_j$ (see Thm.~\myref{thm:V}{a} for notation) as a $k^{(i)}\times k^{(j-1)}$ matrix.

\paragraph{Batch normalization} %
It is possible to apply batch normalization (batchnorm) immediately after ReLU in a $G$-DNN without breaking $G$-invariance. 
That batchnorm is compatible with $G$-DNNs, and in particular type 2 signed perm-reps, out-of-the-box is not obvious, 
and we verify the compatibility in Prop.~\ref{prop:bn} (see App.~\ref{app:remarks}). 
The proof relies on the facts that 
(1) the first $n_i$ columns of $W^{(i)}$ and $V^{(i)}$ are equal and 
(2) the first $n_i$ columns of $V^{(i)}$ sum to zero if $\rho^{(i)}$ is a type 2 signed perm-irrep.

\section{Examples}
\label{sec:examples}

\subsection{Concatenated ReLU}
\label{sec:crelu}

$G$-invariant architectures with signed perm-reps have previously been constructed without skip connections~\citep{cohen2017steerable} by replacing ReLU with the CReLU activation function defined in Eq.~\eqref{eq:crelu}. 
The upshot is that for every signed perm-rep $\rho:G\mapsto\PZ(n)$, CReLU enjoys the equivariance property
\[ \crelu(\rho(g)x) = \pi_{\rho}(g)\crelu(x)\forall g\in G,x\in\RR^n, \]
where $\pi_{\rho}$ is the unraveled ordinary perm-rep appearing in Thm.~\ref{thm:signed2ordinary}; 
CReLU architectures are thus compatible with signed perm-reps. 
The following example, however, establishes that CReLU architectures are strictly special cases of the $G$-DNN architectures presented in this paper.

\begin{example} \label{ex:crelu}
Let $f:\RR^m\mapsto\RR$ be a DNN of the form
\[ f(x) = U^{(d)}\crelu(U^{(d-1)} \cdots \crelu(U^{(1)} x) \cdots), \]
where $U^{(1)}\in\RR^{n_2\times n_1}$ and $U^{(i)}\in\RR^{n_{i+1}\times 2n_i}\forall i\in\{2,\ldots,d\}$ are equivariant weight matrices satisfying
\begin{align*}
\rho^{(1)}(g) U^{(1)} &= U^{(1)} \pi^{(0)}(g)\forall g\in G \\
\rho^{(i)}(g) U^{(i)} = U^{(i)} \pi_{\rho^{(i-1)}}(g)\forall g\in G,i\in\{2,\ldots,d\},
\end{align*}
where $\pi_{\rho^{(i-1)}}$ is the unraveling of $\rho^{(i-1)}$ as defined in Thm.~\ref{thm:signed2ordinary}. 
Then $f$ admits an expression as a $G$-DNN of depth $d$ whose latent weight matrices are given by the recursion
\begin{align*}
V^{(1)} &= U^{(1)} \\
V^{(i+1)} &= \lmat U^{(i+1)}_1+U^{(i+1)}_2 & \frac{1}{2}(U^{(i+1)}_1-U^{(i+1)}_2) V^{(i)} \rmat\forall i\in\{1,\ldots,d-1\}.
\end{align*}
\end{example}

Although Ex.~\ref{ex:crelu} holds for any sequence of signed perm-reps $\rho^{(1)},\ldots,\rho^{(d)}$ with $\rho^{(d)}$ trivial, 
the primary case of interest is when each $\rho^{(i)}$, $i\leq d-1$, is a direct sum only of type 2 signed perm-irreps; 
otherwise, if we had a subset of weight vectors that transformed by a type 1 signed perm-irrep, then we would apply ordinary RELU to their respective preactivations instead of CReLU. 
This is why it is also sufficient to consider CReLU networks without bias vectors, as type 2 signed perm-irreps constrain bias vectors to be zero.

\begin{table}
\centering
\caption{\label{table:admissible-crelu} %
Ratio of the number of admissible CReLU architectures to the total number of architectures for every depth and every group $G$, $|G| = 8$, up to isomorphism. 
Only architectures corresponding to sequences of irreps of strictly decreasing degree are considered.
}
\begin{tabular}{cccccc}
\toprule
Depth & $C_8$ & $C_2\times C_4$ & $C_2^3$ & $D_4$ & $Q_8$ \\
\midrule
2 & 5/5 & 8/15 & 11/43 & 14/21 & 9/9 \\
3 & 8/8 & 30/62 & 88/434 & 65/104 & 20/20 \\
4 & 4/4 & 34/48 & 238/392 & 66/84 & 12/12 \\
\bottomrule
\end{tabular}
\end{table}

Example~\ref{ex:crelu} implies that the $G$-DNN architectures introduced in this paper are at least as powerful as ($G$-invariant) CReLU-based architectures, whose utility has already been demonstrated in the literature~\citep{shang2016understanding, cohen2017steerable}. 
Moreover, we recover CReLU architectures only when the latent weight matrices $V^{(i)}$ take a very special form; 
this suggests that there may exist $G$-DNN architectures that cannot be constructed with CReLU alone. 
To investigate this possibility, we count the numbers of \emph{admissible} CReLU architectures for various groups following the same setup used to obtain the results in Table~\ref{table:admissible}. 
As before, we consider every sequence of signed perm-irreps $\rho_{H^{(1)}K^{(1)}},\ldots,\rho_{H^{(d)}K^{(d)}}$ of strictly decreasing degree terminating with the trivial rep. 
In contrast to $G$-DNNs, however, the criterion for a CReLU architecture to be admissible is much simpler; 
since there are no explicit skip connections, then the function $\phi^{(i+1)}$ in Thm.~\ref{thm:phi} no longer depends recursively on $\phi^{(i)}$. 
The resulting condition for a CReLU architecture with signed perm-irreps $\rho_{(H^{(1)}K^{(1)}},\ldots,\rho_{H^{(d)}K^{(d)}}$ to be admissible is
\[ \theta(H^{(i+1)}, K^{(i+1)}, H^{(i)}) = K^{(i+1)}\forall i\in\{1,\ldots,d-1\}. \]
Based on this, we report fractions of admissible CReLU architectures (Table~\ref{table:admissible-crelu}) in direct analogy to Table~\ref{table:admissible}. 
At depth $4$, there are fewer admissible CReLU architectures than admissible $G$-DNN architectures.

\begin{table}
\centering
\caption{\label{table:admissible-ico} %
Numbers of admissible CReLU and $G$-DNN architectures at every depth for the symmetry group $G$ of the icosahedron. 
Only architectures corresponding to sequences of irreps of strictly decreasing degree are considered. 
We do not report the total number of such sequences as we found every sequence to correspond to an admissible $G$-DNN architecture.
}
\begin{tabular}{ccc}
\toprule
Depth & CReLU & $G$-DNN \\
\midrule
2 & 20 & 20 \\
3 & 136 & 142 \\
4 & 441 & 516 \\
5 & 776 & 1089 \\
6 & 769 & 1392 \\
7 & 407 & 1064 \\
8 & 90 & 448 \\
9 & 0 & 80 \\
Total & 2639 & 4751 \\
\bottomrule
\end{tabular}
\end{table}

As a more striking example, we consider the symmetry group of the icosahedron used later for 3D object classification in Sec.~\ref{sec:3d}. 
We report the numbers of admissible CReLU and $G$-DNN architectures for this group (Table~\ref{table:admissible-ico}). 
There are about $1.8\times$ as many admissible $G$-DNN architectures as there are admissible CRELU architectures, 
including nine-layer $G$-DNNs with no CRELU architectures of the same depth. 
We conclude that while the CRELU activation function gives a simple way to implement equivariance with respect to signed perm-reps, 
our $G$-DNNs grant us access to large parts of $G$-invariant architecture space previously unreachable.

\subsection{Binary multiplication}
\label{sec:prod}

Our first example application is an exact result and demonstrates that $G$-DNNs with type 2 signed perm-reps occur even in the context of simple mathematical operations.

\begin{example} \label{ex:prod}
For $d\geq 3$, let $m = 2^d$ and 
\[ \mathcal{X} = \{x\in \{0, 1\}^m: x_{2i-1}+x_{2i}=1\forall i\in\{1,\ldots,m/2\}\}. \]
For $\{e_1,\ldots,e_m\}$ the standard orthonormal basis, 
Let $G < \operatorname{P}(m)$ be the subgroup of all even permutation matrices that only transpose $e_{2i-1}$ and $e_{2i}$. 
Define the function $f:\mathcal{X}\mapsto \{-1, 1\}$ by
\[ f(x) = \prod_{i=1}^{m/2} (x_{2i-1}-x_{2i}). \]
Then $f$ admits an expression as a $G$-DNN of depth $d$ where:
\begin{enumerate}[label={(\alph*)}]
\item \label{ex:prod:a}
For $i\in\{1,\ldots,d\}$ and $j\in\{1,\ldots,i\}$, the blocks $V^{(i)}_j$ of the latent weight matrices (see Thm.~\myref{thm:V}{a}) are given by
\[ V^{(i)}_j = 
\begin{cases}
I_{m/2^{i+1}}\otimes \lmat 1 & -1 & 1 & -1\\ 1 & -1 & -1 & 1\rmat, & \mbox{ if } 1 \leq i=j < d \\
\lmat 1 & -1\rmat, & \mbox{ if } i=j=d \\
0, & \mbox{ otherwise.}
\end{cases} \]
\item \label{ex:prod:b}
For $i\in\{1,\ldots,d-1\}$, the $i$th layer transforms by the signed perm-rep
\[ \rho^{(i)} = 
\begin{cases}
\bigoplus_{j=1}^{m/2^{i+1}} \rho_{H^{(i)}_j K^{(i)}_j}, & \mbox{ if } i \leq d-2 \\
\rho_{H^{(d-1)}_1 K^{(d-1)}_1}\oplus \rho_{H^{(d-1)}_1 K^{(d-1)}_1}, & \mbox{ if } i=d-1,
\end{cases} \]
where
\begin{align*}
K^{(1)}_j &= \{g\in G: gv_j = v_j\} \\
H^{(1)}_j &= \{g\in G: gv_j = \pm v_j\},
\end{align*}
where $v_j = e_{4j-3}-e_{4j-2}+e_{4j-1}-e_{4j}$, and
\begin{align*}
K^{(i+1)}_j &= H^{(i)}_{2j-1}\cap H^{(i)}_{2j} \\
H^{(i+1)}_j &= (H^{(i)}_{2j-1}\cap H^{(i)}_{2j})\cup ((G\setminus H^{(i)}_{2j-1})\cap (G\setminus H^{(i)}_{2j})),
\end{align*}
for $i\in\{1,\ldots,d-2\}$. 
The rep $\rho^{(d)}$ is trivial as usual.
\end{enumerate}
\end{example}

Here, $f(x)$ is the product of $\frac{m}{2}$ binary elements in $\{-1, 1\}$ 
but where the elements are each represented as 2D one-hot vectors and are then concatenated into the $m$-dimensional input $x$. 
From the perspective of multiplication of elements in $\{-1, 1\}$, 
the group $G$ corresponds to an even number of sign flips, 
which clearly leaves the final product invariant. 
Example~\ref{ex:prod} gives an explicit construction of a $G$-DNN that implements the function $f$. 
Each layer of the $G$-DNN partitions its input into pairs and takes each of their products; 
the network thus iteratively coarsegrains the input until the final scalar product is returned.%
\footnote{%
Although $f$ is a simple function, 
we were unable to construct a $G$-DNN---in particular, a shallower $G$-invariant architecture that (1) has $G$-equivariant preactivations and (2) does not have an exponential number of hidden neurons per layer---simpler than the one given in Ex.~\ref{ex:prod} that fits $f$. 
For example, one could compute $f(x)$ with a shallow network that first maps $x$ linearly into $\{-1, 1\}^{\frac{m}{2}}$, then takes the sum $\vec{1}^{\top} x$, and finally returns the parity of the sum using a suitable combination of ReLU neurons; 
however, the function $x\rightarrow \vec{1}^{\top} x$ is not $G$-equivariant.%
} 
Each latent weight matrix of the $G$-DNN has a block-diagonal structure and no latent skip connections; 
however, by Thm.~\myref{thm:V}{a}, the $G$-DNN does have skip connections in terms of the apparent weight matrices. 
All signed perm-irreps in the network--except $\rho^{(d)}$ which must be trivial--are type 2, 
and the image of each one is isomorphic to the Klein-4 group. 
The image of the penultimate rep $\rho^{(d-1)}$ is also isomorphic to the Klein-4 group, 
but the rep itself is not irreducible and decomposes into two copies of a type 2 scalar irrep.

\begin{table}
\centering
\caption{\label{table:prod} %
Binary cross-entropy losses (mean and standard deviation over $24$ random initialization seeds) of four $G$-DNN architectures (see main text) on the binary multiplication problem. 
We report both the training loss and validation loss before training (initial) and after $5$ epochs of training (final). 
Training and validation losses are equal because each of the two class labels corresponds to exactly one $G$-orbit, 
over which each $G$-DNN is constant-- 
regardless of the dataset split. 
All values correspond to a classification accuracy of $50\%$, 
except the final losses of zero for the type 2 architecture corresponding to 100\% accuracy.
}
\begin{tabular}{cS[table-format=2.2(2)]S[table-format=2.2(2)]SS}
\toprule
Architecture & {Initial train} & {Initial val} & {Final train} & {Final val} \\
\midrule
Type 1 & 5.11\pm 4.68 & 5.11\pm 4.68 & 0.71\pm 0.04 & 0.71\pm 0.04 \\
\textbf{Type 2} & 1.33\pm 0.60 & 1.33\pm 0.60 & 0.00\pm 0.00 & 0.00\pm 0.00 \\
Unraveled (random init) & 11.36\pm 12.34 & 11.36\pm 12.34 & 0.71\pm 0.03 & 0.71\pm 0.03 \\
Unraveled (type 2 init) & 1.33\pm 0.60 & 1.33\pm 0.60 & 0.70\pm 0.00 & 0.70\pm 0.00 \\
\bottomrule
\end{tabular}
\end{table}

We investigate Ex.~\ref{ex:prod} empirically by generating the complete dataset $\{(x, f(x)): x\in \mathcal{X}\}$ for $m=16$. 
Since $f(x)=\pm 1$, then we regard the estimation of $f$ as a binary classification problem. 
We use a random 20\% of the dataset (with class stratification) for training and the rest for validation. 
We instantiate the ``type 2'' architecture with the type 2 signed perm-irreps in Ex.~\myref{ex:prod}{b} and randomly initialize its weights. 
We compare the type 2 architecture to two type 1 baselines: 
(1) the ``type 1'' architecture obtained by sending%
\footnote{This transformation is called topological tunneling in the literature~\citep{agrawal2022classification}} %
$\rho_{HK}\rightarrow \rho_{HH}$ (see Thm.~\ref{thm:orthogonal}) and 
(2) the ``unraveled'' architecture obtained by sending $\rho_{HK}\rightarrow \rho_{KK}$ (seeThm.~\myref{thm:signed2ordinary}{d}). 
The type 1, type 2, and unraveled architectures are thus analogous to the three weightsharing patterns in Fig.~\ref{fig:patterns} (see Sec.~\ref{sec:central}); 
i.e., the type 1 architecture has the same number of hidden neurons as the type 2 architecture but expresses different functions, 
and the unraveled architecture has double the number of hidden neurons and the capacity to express both the smaller two architectures as well as much more.

We trained all architectures for $5$ epochs with the Adam optimizer with minibatch size $64$, learning rate $0.01$, and learning rate decay $0.99$ per step. 
Starting at 50\% classification accuracy, the type 2 architecture quickly achieves 100\% training and validation accuracy as well as a final binary cross-entropy loss of $0.00$ (Table~\ref{table:prod}). 
All other architectures remain stuck at 50\% accuracy even after training. 
To explain these results, we first note that in the type 2 architecture, the latent weight matrices for $i\in\{1,\ldots,d-1\}$ are constrained to have the form
\[ V^{(i)}_j = 
\begin{cases}
I_{m/2^{i+1}}\otimes \lmat u^{(i)}_j & -u^{(i)}_j & v^{(i)}_j & -v^{(i)}_j \\ u^{(i)}_j & -u^{(i)}_j & -v^{(i)}_j & v^{(i)}_j \rmat, & \mbox{ if } 1 \leq j=i < d \\
0, & \mbox{ if } 1 \leq j < i < d,
\end{cases} \]
where the $u^{(i)}_j$ and $v^{(i)}_j$ are free learnable parameters. 
In the type 1 architecture, the negative signs in front of $v^{(i)}_j$ are omitted, 
and the off-diagonal blocks are no longer zero but are instead constrained in the same way as the diagonal blocks. 
As a result, the very first linear transformation $g^{(1)}(x) = V^{(1)}x$ in the architecture is invariant to every transposition $(x_{2i-1}, x_{2i})\rightarrow (x_{2i}, x_{2i-1})$-- not just even numbers of them; 
the rest of the architecture is thus unable to distinguish the two classes of inputs, resulting in a flat $50\%$ accuracy. 
The type 1 architecture simply does not have the capacity to solve the binary classification problem.

In contrast to the type1 and type 2 architectures, the latent weight matrices $V^{(i)}_i$ of the unraveled architectures are not even constrained to be block-diagonal, 
and the $V^{(i)}_j$ for $j<i$ are not constrained to be zero. 
The unraveled architecture thus has about $6.5\times$ the number of free learnable parameters compared to the type 2 architecture, 
and it must learn the block-diagonal weight structure from data alone. 
Based on the training loss, the complete failure of the unraveled architecture is likely due to a severe trainability issue. 
To confirm the unraveled architecture indeed has the capacity to express the solution and in particular the type 2 architecture, 
we loaded the trained weights of the type 2 architecture into the unraveled architecture by an appropriate mapping and confirmed that the resulting architecture indeed achieves 100\% training and validation accuracy. 
To test if the failure to train is due to poor initialization, we implemented the unraveled architecture with both random initialization (called ``random init'' in Table~\ref{table:prod}) and with the randomly initialized weights of the type 2 architecture (called ``type 2 init'' in Table~\ref{table:prod}). 
While initialization with the initial weights of the type 2 architecture reduces the initial loss of the unraveled architecture to match the type 2 architecture, it does not solve the trainability issue. 
We thus conclude the unraveled architecture simply does not have sufficient inductive bias to solve the binary classification problem.

\subsection{3D object classification}
\label{sec:3d}

\begin{wrapfigure}{R}{0.5\textwidth}
\centering
\includegraphics[trim={0 0 0 0.75cm}, clip, width=0.4\textwidth]{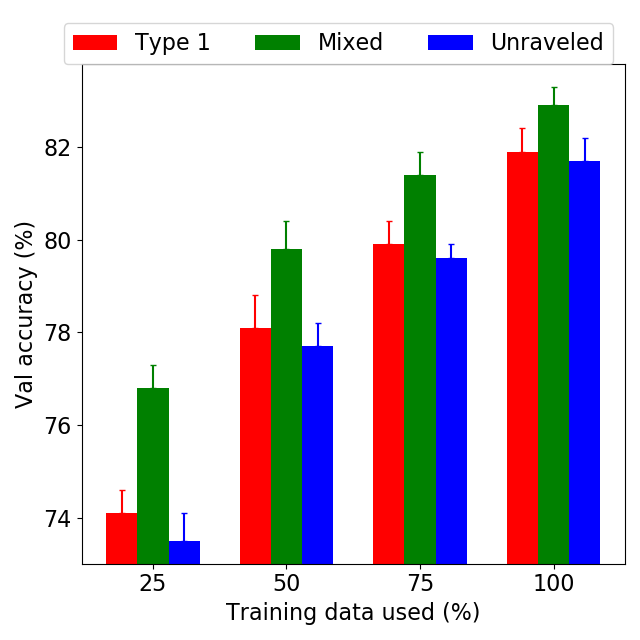}
\caption{\label{fig:mn40} %
Validation classification accuracies (mean and standard deviation over $24$ random initialization seeds) for three $G$-DNN architectures (see main text) on the ModelNet40 dataset using different percentages of the full training data. 
The mixed architecture---the only one including type 2 signed perm-irreps---clearly outperforms the two baselines, 
with the performance gap increasing as less training data is available.
}
\end{wrapfigure}

Our second example application demonstrates that $G$-DNNs with type 2 signed perm-reps carry inductive bias that can be useful ``in the wild''. 
We consider the ModelNet40 dataset\citep{wu20153d}, which contains $9843$ training and $2468$ validation samples of 3D CAD mesh representations of $40$ different objects ranging from airplanes to toilets. 
The problem is to predict the object class given an input mesh. 
We preprocess the data in identical fashion to \citet{jiang2019spherical}.%
\footnote{%
\citet{jiang2019spherical} perform five iterations of refinement to increase resolution, whereas we only perform two. 
This is because our current $G$-DNN implementation is based on fully-connected linear layers and thus does not scale well to high resolution. 
We plan to remedy this in future work by introducing local kernel windows. 
We attribute the difference in accuracy reported in Fig.~\ref{fig:mn40} and that reported by \citet{jiang2019spherical} to this difference in resolution.
} %
Specifically, we first bound each mesh in the unit sphere and discretize the sphere into an icosahedron. 
To increase resolution, we include the midpoint on each icosahedral edge as a vertex, normalize all vertices to have unit norm, and then repeat once more on the new polyhedron; 
this yields $162$ vertices in all. 
From each of these vertices, we perform ray-tracing to the origin; 
we record the distance from the sphere to the mesh as well as the sine and cosine of the incident angle. 
The representation is further augmented with the three channels corresponding to the convex hull of the input mesh, yielding a $6$-channel input representation over $162$ points.

The relevant symmetry group $G$ is the symmetry group of the icosahedron, 
generated by (1) the order-$3$ rotation about the normal vector to one of the faces 
and (2) the order-$5$ rotation about one of the vertices; 
the group has $60$ elements and is isomorphic to the alternating group $A_5$. 
We consider $G$-DNN architectures with four layers and $16$, $32$, $64$, and $40$ output channels per rep in each respective layer. 
We specify apriori the degrees of the signed perm-irreps to be used in each layer as $30$, $15$, $10$, and $1$; 
we then enumerate the admissible architectures as described after Thm.~\ref{thm:phi}, and we select the following designs: 
For the ``mixed'' architecture, we select one type 1 and one type 2 signed perm-irrep in each of the first three layers%
\footnote{In each layer, the selected type 1 and type 2 signed perm-irreps are cohomologous as defined in \citet{agrawal2022classification}; 
i.e., they have the forms $\rho_{HH}$ and $\rho_{HK}$ respectively.%
}%
we select the trivial rep in the last layer, as always. 
We compare the mixed architecture to the corresponding ``type 1'' and ``unraveled'' architectures, 
which are constructed exactly as done in Sec.~\ref{sec:prod}. 

We trained all architectures for $500$ epochs with the Adam optimizer with minibatch size $64$, learning rate $0.01$, and learning rate decay $0.99$ per step. 
We included batchnorm as described in Sec.~\ref{sec:remarks} after each ReLU layer. 
In each run, we performed retroactive early stopping by recording the highest validation accuracy achieved over all epochs. 
We find that the mixed architecture---the only one containing type 2 signed perm-irreps---significantly outperforms the two baseline architectures in terms of validation accuracy, 
with the performance gap increasing as less training data is used. 
This suggests that the mixed architecture carries stronger inductive bias, still consistent with the ground truth, as compared to the baselines.

\section{Conclusion}
\label{sec:conclusion}

We have introduced the $G$-DNN, a $G$-invariant densely connected DNN architecture. 
In contrast to previous $G$-invariant architectures in the literature such as the $G$-CNN~\citep{cohen2016group}, the $G$-DNN is built with \emph{signed} perm-reps that do not require individual layers of the network to be $G$-equivariant. 
The result is a richer family of $G$-invariant architectures never seen before (carefully quantified in terms of admissible architecture count), 
and we have demonstrated with both theoretical and empirical examples that some of these novel architectures can boost predictive performance.

To be clear, we do not claim the $G$-DNN to be a new state-of-the-art (SOTA) for $G$-invariant deep learning. 
Rather, our work is a demonstration that signed perm-reps, combined with skip connections, are mathematically natural building blocks for deep $G$-invariant architectures with practical potential. 
Indeed, we suspect that the structures and ideas presented in this paper, once extended and combined with domain-specific bells and whistles, could in fact help to boost the performance of SOTA $G$-invariant architectures.

Even at the domain-agnostic level, however, several open questions remain for future research. 
First, can we extend $G$-DNNs from $G$-invariance to $G$-equivariance? 
The only obstacle here is the construction of architectures guaranteed to be admissible. 
Second, can we perform $G$-NAS to find the optimal signed perm-irreps to use in each layer? 
Third and finally, are there ways of enforcing $G$-invariance that go even beyond the $G$-DNN architectures described in this paper? 
A complete classification of \emph{all} $G$-invariant architectures would give us the finest possible control on inductive bias, thereby allowing us---in principle---to optimize predictive performance on $G$-invariant problems.


\acks{%
D.A. was supported by NSF award No.~2202990. 
J.O. was supported by DOE grant DE-SC0018175.%
}

\appendix
\section{Signed permutation representations}
\label{app:rho}

\subsection{Central hypothesis}

\begin{proof}[Proof of Thm.~\ref{thm:orthogonal}]
For every $J\leq G$, let $P_J$ be the $m\times m$ orthogonal projection operator onto the subspace of $\RR^m$ pointwise-invariant under the action of $J$:
\[ P_J = \frac{1}{|J|}\sum_{g\in J} g. \]
By Thm.~4b of \citet{agrawal2022classification}, there exists $w_2\in\ran(P_{K_2}-(|H:K_2|-1)P_H)$ and a transversal $\{g_1,\ldots,g_n\}$ of $G/H$ such that
\[ W_2 = \lmat g_1w_2 & g_2w_2 & \cdots & g_nw_2 \rmat^{\top}. \]
Moreover, there exists $w_1\in\ran(P_{K_1}-(|H:K_1|-1)P_H)$ and $P\in\operatorname{P}(n)$ such that
\[ PW_1 = \lmat g_1w_1 & g_2w_1 & \cdots & g_nw_1 \rmat^{\top}. \]
(The permutation matrix $P$ is in general necessary so we can use the same transversal of $G/H$.) 
We have
\begin{align*}
\diag(PW_1W_2^{\top}) 
&= \{(g_iw_1)^{\top}g_iw_2\}_{i=1}^n \\
&= \{w_1^{\top}g_i^{\top}g_iw_2\}_{i=1}^n \\
&= \{w_1^{\top}w_2\}_{i=1}^n \\
&= 0,
\end{align*}
where the last step follows by Prop.~6 of \citet{agrawal2022classification}.
\end{proof}

\begin{proof}[Proof of Thm.~\ref{thm:signed2ordinary}]
\textbf{(a) } %
Let $\pi:G\mapsto\operatorname{P}(n)$ and $\zeta:G\mapsto\operatorname{Z}(n)$ be the unique functions satisfying $\rho(g) = \pi(g)\zeta(g)\forall g\in G$ 
(note $\pi$ should not be confused with $\pi_{\rho}$ appearing in the theorem statement). 
Evaluating the Kronecker product, 
the function $\pi_{\rho}$ can be rewritten as
\begin{align*}
\pi_{\rho}(g) 
&= \heavi\left(\lmat \rho(g) & -\rho(g)\\ -\rho(g) & \rho(g)\rmat\right) \\
&= \heavi\left(\lmat \pi(g)\zeta(g) & -\pi(g)\zeta(g)\\ -\pi(g)\zeta(g) & \pi(g)\zeta(g)\rmat\right) \\
&= \heavi\left(\lmat \pi(g) & 0\\ 0 & \pi(g)\rmat \lmat \zeta(g) & -\zeta(g)\\ -\zeta(g) & \zeta(g)\rmat\right) \\
&= \lmat \pi(g) & 0\\ 0 & \pi(g)\rmat \heavi\left(\lmat \zeta(g) & -\zeta(g)\\ -\zeta(g) & \zeta(g)\rmat\right),
\end{align*}
where in the last step we exploited the permutation-equivariance of elementwise operations. 
Now since $\heavi(z) = \frac{1+z}{2}$ for $z\in\{-1, 1\}$, then we have
\begin{align*}
\pi_{\rho}(g) 
&= \lmat \pi(g) & 0\\ 0 & \pi(g)\rmat \frac{1}{2}\left(\lmat I_n & I_n\\ I_n & I_n\rmat + \lmat \zeta(g) & -\zeta(g)\\ -\zeta(g) & \zeta(g)\rmat\right) \\
&= \frac{1}{2}\left(\lmat \pi(g) & \pi(g)\\ \pi(g) & \pi(g)\rmat + \lmat \pi(g)\zeta(g) & -\pi(g)\zeta(g)\\ -\pi(g)\zeta(g) & \pi(g)\zeta(g)\rmat\right) \\
&= \frac{1}{2}\lmat \pi(g)+\rho(g) & \pi(g)-\rho(g)\\ \pi(g)-\rho(g) & \pi(g)+\rho(g)\rmat. (*)
\end{align*}

Now, let $\pi(g)_{ij}$ (resp. $\rho(g)_{ij}$) be the unique nonzero element in the $i$th row of $\pi(g)$ (resp. $\rho(g)$). 
Since $\rho(g)_{ij} = \pm \pi(g)_{ij}$, then exactly one of $\frac{1}{2}[\pi(g)_{ij} \pm \rho(g)_{ij}]$ is unity, 
while the other is zero. 
Thus, every row of (*) is a one-hot vector. 
The same argument can be made for the columns, and hence (*) is a bona fide permutation matrix; 
i.e., $\pi_{\rho}:G\mapsto\operatorname{P}(2n)$ is a well-defined function.

All that remains is to show $\pi_{\rho}$ is a homomorphism. 
We rewrite (*) as
\[ \pi_{\rho}(g) = \frac{1}{2}\left(\lmat I_n\\ I_n\rmat \pi(g) \lmat I_n & I_n\rmat + \lmat I_n\\ -I_n\rmat \rho(g) \lmat I_n & -I_n\rmat\right). \]
Then for $g,h\in G$, it is easy to verify that
\begin{align*}
\pi_{\rho}(g)\pi_{\rho}(h) 
&= \frac{1}{4}\left(\lmat I_n\\ I_n\rmat \pi(g) \lmat I_n & I_n\rmat + \lmat I_n\\ -I_n\rmat \rho(g) \lmat I_n & -I_n\rmat\right) \left(\lmat I_n\\ I_n\rmat \pi(h) \lmat I_n & I_n\rmat \right. \\
&\quad + \left. \lmat I_n\\ -I_n\rmat \rho(h) \lmat I_n & -I_n\rmat\right) \\
&= \frac{1}{4}\left(2\lmat I_n\\ I_n\rmat \pi(g)\pi(h) \lmat I_n & I_n\rmat + 0 + 0 + 2\lmat I_n\\ -I_n\rmat \rho(g)\rho(h) \lmat I_n & -I_n\rmat\right) \\
&= \pi_{\rho}(gh).
\end{align*}

\textbf{(b) } %
Using (*) from part (a), we have
\begin{align*}
\pi_{\rho}(g) \lmat W\\ -W\rmat 
&= \frac{1}{2}\lmat \pi(g)+\rho(g) & \pi(g)-\rho(g)\\ \pi(g)-\rho(g) & \pi(g)+\rho(g) \rmat \lmat W\\ -W\rmat \\
&= \frac{1}{2}\lmat \pi(g)W+\rho(g)W -\pi(g)W+\rho(g)W \\ \pi(g)W-\rho(g)W -\pi(g)W-\rho(g)W \rmat \\
&= \frac{1}{2}\lmat 2\rho(g)W \\ -2\rho(g)W \rmat \\
&= \rho(g) \lmat W\\ -W\rmat \\
&= \lmat W\\ -W\rmat g,
\end{align*}
proving the claim.

\textbf{(c) } %
For the forward implication, we prove its contrapositive; 
suppose $\rho$ is type 1. 
Then $\pi = \rho$ so that (*) implies
\[ \pi_{\rho} = \frac{1}{2}\lmat 2\pi(g) & 0\\ 0 & 2\pi(g)\rmat = \lmat\pi(g) & 0\\ 0 & \pi(g)\rmat, \]
which is clearly reducible.

For the reverse implication, suppose $\rho$ is type 2. 
To show $\pi_{\rho}$ is irreducible, we will show it is transitive on the standard orthonormal basis $\{\omega_1,\ldots,\omega_{2n}\}$. 
Let $i,j\in\{1,\ldots,2n\}$, 
and without loss of generality suppose $i\leq n$.

\textbf{Case 1: } %
Suppose $j\leq n$. 
Then $\omega_i$ is just $e_i$ (the $i$th standard orthonormal basis vector of dimension $n$) concatenated with $\vec{0}_n$, 
and similar for $\omega_j$. 
Now since $\rho$ is irreducible, then there exists $g\in G$ such that $\rho(g)e_i = e_j$; 
note $\pi(g)e_i = e_j$ as well. 
We thus have
\begin{align*}
\pi_{\rho}(g)\omega_i 
&= \frac{1}{2}\lmat\pi(g)+\rho(g) & \pi(g)-\rho(g)\\ \pi(g)-\rho(g) & \pi(g)+\rho(g)\rmat \lmat e_i\\ \vec{0}_n\rmat \\
&= \frac{1}{2}\lmat \pi(g)e_i+\rho(g)e_i \\ \pi(g)e_i-\rho(g)e_i\rmat \\
&= \frac{1}{2}\lmat e_j+e_j\\ e_j-e_j\rmat \\
&= \lmat e_j\\ \vec{0}_n\rmat \\
&= \omega_j,
\end{align*}
establishing transitivity in this case.

\textbf{Case 2: } %
Suppose instead $j>n$. 
Then $\omega_j$ is the concatenation of $\vec{0}_n$ and $e_{j-n}$. 
Since $\rho$ is irreducible, then there exists $g\in G$ such that $\rho(g)e_i = -e_{n-j}$. 
The rest of the proof proceeds in analogy to case 1.

\textbf{(d) } %
For every $v\in\RR^n$ and linear rep $\tau:G\mapsto\operatorname{GL}(n, \RR)$, define the stabilizer subgroup
\[ \st_{\tau}(v) = \{g\in G: \tau(g)v = v\}. \]
Since $\rho = \rho_{HK}$, then $\st_{\rho}(e_1) = K$. 
Since $\pi_{\rho}$ is an ordinary perm-rep, then all we must show is $\st_{\pi_{\rho}}(\omega_1) = K$ to establish the claim. 
Similar to part (c), $\st_{\pi_{\rho}}(\omega_1)$ is the set of all $g\in G$ such that
\begin{align*}
\pi_{\rho}(g)\omega_1 &= \omega_1 \\
\frac{1}{2}\lmat\pi(g)+\rho(g) & \pi(g)-\rho(g)\\ \pi(g)-\rho(g) & \pi(g)+\rho(g)\rmat \lmat e_1\\ \vec{0}_n\rmat &= \lmat e_1\\ \vec{0}_n\rmat \\
\frac{1}{2}\lmat \pi(g)e_1+\rho(g)e_1\\ \pi(g)e_1-\rho(g)e_1\rmat &= \lmat e_1\\ \vec{0}_n\rmat.
\end{align*}
Taking the difference of the two rows, we obtain $\rho(g)e_1 = e_1$; 
hence, $\st_{\pi_{\rho}}(\omega_1) = \st_{\rho}(e_1) = K$, completing the proof.
\end{proof}

\section{\texorpdfstring{$G$}{G}-invariant deep neural networks}
\label{app:gdnn}

\subsection{Parameterization redundancies}
\label{app:parameter}

To understand the inclusion of skip connections as a reparameterization, we rewrite Eq.~\ref{eq:fi} as
\begin{align*}
f^{(i+1)}(x) 
&= 
\lmat
\relu(W^{(i)} f^{(i)}(x)+b^{(i)}) \\
\relu(f^{(i)}(x)) - \relu(-f^{(i)}(x))
\rmat \\
&= 
\lmat
I_{n_{i+1}} & 0 & 0 \\
0 & I_{N_i} & -I_{N_i}
\rmat 
\relu\left(
\lmat
W^{(i)} \\
I_{N_i} \\
-I_{N_i}
\rmat
f^{(i)}(x) + 
\lmat
b^{(i)} \\
0 \\
0
\rmat
\right).
\end{align*}
The outer matrix in the last equation can be combined with the matrix in the next layer; 
the result is a DNN having the same depth as the original---and representing the same input-output function---but with no skip connections, as they have been transformed into additional ReLU neurons.

\begin{proof}[Proof of Prop.~\ref{prop:CPZ}]
We will show $W^{(i+1)} f^{(i+1)}(x) + b^{(i+1)}$ is invariant under the transformation. 
The function $f^{(i+1)}$ transforms as
\begin{align*}
f^{(i+1)}(x) 
&\rightarrow 
\lmat \relu(CPZ W^{(i)}f^{(i)}(x) + CPZ b^{(i)}) \\ f^{(i)}(x) \rmat \\
&= 
\lmat CP & 0 \\ 0 & I_{N_i} \rmat
\lmat \relu(Z(W^{(i)}f^{(i)}(x)+b^{(i)}) \\ f^{(i)}(x) \rmat \\
&= 
\lmat CP & 0 \\ 0 & I_{N_i} \rmat 
\left( \lmat \relu(W^{(i)}f^{(i)}(x)+b^{(i)}) \\ f^{(i)}(x) \rmat 
- \lmat \heavi(-Z)(W^{(i)}f^{(i)}(x) + b^{(i)}) \\ 0 \rmat \right) \\
&= 
\lmat CP & 0 \\ 0 & I_{N_i} \rmat 
\left( \lmat I_{n_{i+1}} & -\heavi(-Z)W^{(i)} \\ 0 & I_{N_i} \rmat 
\lmat \relu(W^{(i)}f^{(i)}(x)+b^{(i)}) \\ f^{(i)}(x) \rmat 
- \lmat \heavi(-Z)b^{(i)} \\ 0 \rmat \right) \\
&= 
\lmat CP & 0 \\ 0 & I_{N_i} \rmat 
\left( \lmat I_{n_{i+1}} & -\heavi(-Z)W^{(i)} \\ 0 & I_{N_i} \rmat 
\lmat \relu(W^{(i)}f^{(i)}(x)+b^{(i)}) \\ f^{(i)}(x) \rmat \right. \\
&\quad - \left. \lmat I_{n_{i+1}} & -\heavi(-Z)W^{(i)} \\ 0 & I_{N_i} \rmat \lmat \heavi(-Z)b^{(i)} \\ 0 \rmat \right) \\
&= 
\lmat CP & 0 \\ 0 & I_{N_i} \rmat 
\lmat I_{n_{i+1}} & -\heavi(-Z)W^{(i)} \\ 0 & I_{N_i} \rmat 
\left( \lmat \relu(W^{(i)}f^{(i)}(x)+b^{(i)}) \\ f^{(i)}(x) \rmat - \lmat \heavi(-Z)b^{(i)} \\ 0 \rmat \right) \\
&= 
\lmat CP & -CP\heavi(-Z)W^{(i)} \\ 0 & I_{N_i} \rmat 
\left( f^{(i+1)}(x) - \lmat \heavi(-Z)b^{(i)} \\ 0 \rmat \right).
\end{align*}
We thus have
\begin{align*}
W^{(i+1)}f^{(i+1)}(x)+b^{(i+1)} 
&\rightarrow 
W^{(i+1)} \lmat (CP)^{-1} & \heavi(-Z)W^{(i)} \\ 0 & I_{N_i} \rmat 
\lmat CP & -CP\heavi(-Z)W^{(i)} \\ 0 & I_{N_i} \rmat 
\left( f^{(i+1)}(x) \right. \\
&\quad - \left. \lmat \heavi(-Z)b^{(i)} \\ 0 \rmat \right) 
+ b^{(i+1)} + W^{(i+1)}\lmat \heavi(-Z)b^{(i)} \\ 0 \rmat \\
&= 
W^{(i+1)} 
\left( f^{(i+1)}(x) - \lmat \heavi(-Z)b^{(i)} \\ 0 \rmat \right) 
+ W^{(i+1)}\lmat \heavi(-Z)b^{(i)} \\ 0 \rmat + b^{(i+1)} \\
&= 
W^{(i+1)} f^{(i+1)}(x) + b^{(i+1)}.
\end{align*}
\end{proof}

\subsection{\texorpdfstring{$G$}{G}-invariant architectures}
\label{app:equivariant}

\begin{proof}[Proof of Lemma~\ref{lemma:psi}]
Since $f^{(1)}$ and $\psi^{(1)}$ are both the identity functions, then the claim is immediate for $i=1$. 
Suppose the claim is true for some $i\in\{1,\ldots,d-1\}$. 
We have for all $g\in G$ and $x\in\RR^m$:
\begin{align*}
f^{(i+1)}(gx) 
&= \lmat \relu(W^{(i)}f^{(i)}(gx)+b^{(i)}) \\ f^{(i)}(gx) \rmat \\
&= \lmat \relu(W^{(i)}\psi^{(i)}(g)f^{(i)}(x)+b^{(i)}) \\ \psi^{(i)}(g)f^{(i)}(x) \rmat \\
&= \lmat \relu(\rho^{(i)}(g)W^{(i)}f^{(i)}(x)+b^{(i)}) \\ \psi^{(i)}(g)f^{(i)}(x) \rmat \\
&= \lmat \relu(\rho^{(i)}(g)W^{(i)}f^{(i)}(x)+\rho^{(i)}(g)b^{(i)}) \\ \psi^{(i)}(g)f^{(i)}(x) \rmat \\
&= \lmat \pi^{(i)}(g)\relu(\zeta^{(i)}(g)(W^{(i)}f^{(i)}(x)+b^{(i)})) \\ \psi^{(i)}(g)f^{(i)}(x) \rmat \\
&= \lmat \pi^{(i)}(g)\relu(W^{(i)}f^{(i)}(x)+b^{(i)}) \\ \psi^{(i)}(g)f^{(i)}(x) \rmat
- \lmat \pi^{(i)}(g)\heavi(-\zeta^{(i)}(g))W^{(i)}f^{(i)}(x) \\ 0 \rmat \\
&\quad - \lmat \pi^{(i)}(g)\heavi(-\zeta^{(i)}(g))b^{(i)} \\ 0 \rmat \\
&= \lmat \pi^{(i)}(g) & -\pi^{(i)}(g)\heavi(-\zeta^{(i)}(g))W^{(i)} \\ 0 & I_{N_i} \rmat f^{(i+1)}(x) 
+ \lmat -\pi^{(i)}(g)\heavi(-\zeta^{(i)}(g))b^{(i)} \\ 0 \rmat.
\end{align*}
Note that
\begin{align*}
-\pi^{(i)}(g)\heavi(-\zeta^{(i)}(g)) 
&= -\frac{1}{2}\pi^{(i)}(g)(I_{n_{i+1}}-\zeta^{(i)}(g)) \\
&= \frac{1}{2}\pi^{(i)}(g)(\zeta^{(i)}(g)-I_{n_{i+1}}) \\
&= \frac{1}{2}(\rho^{(i)}(g) - \pi^{(i)}(g)).
\end{align*}
We thus have
\begin{align*}
-\pi^{(i)}(g)\heavi(-\zeta^{(i)}(g))W^{(i)} 
&= \frac{1}{2}(\rho^{(i)}(g) - \pi^{(i)}(g))W^{(i)} \\
&= \frac{1}{2}(\rho^{(i)}(g)W^{(i)} - \pi^{(i)}(g)W^{(i)}) \\
&= \frac{1}{2}(W^{(i)}\psi^{(i)}(g) - \pi^{(i)}(g)W^{(i)}).
\end{align*}
For the bias term, let $\rho^{(i)} = \rho^{(i)}_1\oplus\cdots \oplus\rho^{(i)}_k$ be the decomposition of $\rho^{(i)}$ into irreducibles; 
decompose $\pi^{(i)}$ and $b^{(i)}$ correspondingly. 
Then we have
\begin{align*}
-\pi^{(i)}(g)\heavi(-\zeta^{(i)}(g))b^{(i)} 
&= \frac{1}{2}(\rho^{(i)}(g) - \pi^{(i)}(g)) b^{(i)} \\
&= \frac{1}{2}\bigoplus_{j=1}^k (\rho^{(i)}_j(g) - \pi^{(i)}_j(g))b^{(i)}_j.
\end{align*}
If $\rho^{(i)}_j$ is type 1, then $\rho^{(i)}_j = \pi^{(i)}_j$ so that the $j$th term in the above direct sum is zero. 
On the other hand, if $\rho^{(i)}_j$ is type 2, then $\rho^{(i)}_j(g)b^{(i)}_j = b^{(i)}_j$ implies $b^{(i)}_j=0$, 
so that again the $j$th summand is zero. 
Therefore, $-\pi^{(i)}(g)\heavi(-\zeta^{(i)}(g))b^{(i)} = 0\forall g\in G$. 
We thus have
\begin{align*}
f^{(i+1)}(gx) 
&= \lmat \pi^{(i)}(g) & \frac{1}{2}(W^{(i)}\psi^{(i)}(g) - \pi^{(i)}(g)W^{(i)}) \\ 0 & I_{N_i} \rmat f^{(i+1)}(x) + 0 \\
&= \psi^{(i+1)}(g) f^{(i+1)}(x).
\end{align*}
The conclusion follows by induction.
\end{proof}

\begin{proof}[Proof of Lemma~\ref{lemma:unfold}]
We will verify that $\psi^{(i+1)}$ defined as claimed satisfies the recursion in Lemma~\ref{lemma:psi}. 
For $i=1$, we have
\begin{align*}
\psi^{(2)}(g) 
&= A^{(1)-1} \Pi^{(1)}(g) A^{(1)} \\
&= \lmat I_{n_2} & -\frac{1}{2}W^{(1)} \\ 0 & I_{n_1} \rmat^{-1} \lmat \pi^{(1)}(g) & 0 \\ 0 & g \rmat \lmat I_{n_2} & -\frac{1}{2}W^{(1)} \\ 0 & I_{n_1} \rmat \\
&= \lmat I_{n_2} & \frac{1}{2}W^{(1)} \\ 0 & I_{n_1} \rmat \lmat \pi^{(1)}(g) & 0 \\ 0 & g \rmat \lmat I_{n_2} & -\frac{1}{2}W^{(1)} \\ 0 & I_{n_1} \rmat \\
&= \lmat \pi^{(1)}(g) & \frac{1}{2}(W^{(1)}g - \pi^{(1)}(g)W^{(1)}) \\ 0 & g \rmat \\
&= \lmat \pi^{(1)}(g) & \frac{1}{2}(W^{(1)}\psi^{(1)}(g) - \pi^{(1)}(g)W^{(1)}) \\ 0 & \psi^{(1)}(g) \rmat,
\end{align*}
which indeed agrees with Lemma~\ref{lemma:psi}.

Now suppose the claim holds for some $i\in\{2,\ldots,d-1\}$. 
Observe that
\begin{align*}
A^{(i)} &= \lmat I_{n_{i+1}} & -\frac{1}{2}W^{(i)} \\ 0 & A^{(i-1)} \rmat \\
\Pi^{(i)}(g) &= \lmat \pi^{(i)}(g) & 0 \\ 0 & \Pi^{(i-1)}(g) \rmat.
\end{align*}
We thus have
\begin{align*}
\psi^{(i+1)}(g) 
&= A^{(i)-1} \Pi^{(i)}(g) A^{(i)} \\
&= 
\lmat I_{n_{i+1}} & -\frac{1}{2}W^{(i)} \\ 0 & A^{(i-1)} \rmat^{-1} 
\lmat \pi^{(i)}(g) & 0 \\ 0 & \Pi^{(i-1)}(g) \rmat 
\lmat I_{n_{i+1}} & -\frac{1}{2}W^{(i)} \\ 0 & A^{(i-1)} \rmat \\
&= 
\lmat I_{n_{i+1}} & \frac{1}{2}W^{(i)}A^{(i-1)-1} \\ 0 & A^{(i-1)-1} \rmat 
\lmat \pi^{(i)}(g) & 0 \\ 0 & \Pi^{(i-1)}(g) \rmat 
\lmat I_{n_{i+1}} & -\frac{1}{2}W^{(i)} \\ 0 & A^{(i-1)} \rmat \\
&= 
\lmat \pi^{(i)}(g) & \frac{1}{2}(W^{(i)}A^{(i-1)-1}\Pi^{(i-1)}(g)A^{(i-1)} - \pi^{(i)}(g)W^{(i)}) \\ 0 & A^{(i-1)-1}\Pi^{(i-1)}(g)A^{(i-1)} \rmat \\
&= \lmat \pi^{(i)}(g) & \frac{1}{2}(W^{(i)}\psi^{(i)}(g) - \pi^{(i)}(g)W^{(i)}) \\ 0 & \psi^{(i)}(g) \rmat.
\end{align*}
The conclusion follows by induction.
\end{proof}

\begin{proof}[Proof of Thm.~\ref{thm:V}]
\textbf{(a) } %
By Lemmas~\ref{lemma:psi}-\ref{lemma:unfold}, we have
\begin{align*}
\rho^{(i)}(g)W^{(i)} &= W^{(i)}\psi^{(i)}(g) \\
\rho^{(i)}(g)W^{(i)} &= W^{(i)}A^{(i-1)-1}\Pi^{(i-1)}(g)A^{(i-1)} \\
\rho^{(i)}(g)W^{(i)}A^{(i-1)-1} &= W^{(i)}A^{(i-1)-1}\Pi^{(i-1)}(g).
\end{align*}
Let $V^{(i)} = W^{(i)}A^{(i-1)-1}$. 
Thus, $W^{(i)} = V^{(i)}A^{(i-1)}$.

All that is left is to establish equivariance of the blocks of $V^{(i)}$. 
We have
\begin{align*}
\rho^{(i)}(g)V^{(i)} 
&= \rho^{(i)}(g)W^{(i)}A^{(i-1)-1} \\
&= W^{(i)}\psi^{(i)}(g)A^{(i-1)-1} \\
&= W^{(i)}A^{(i-1)-1}\Pi^{(i-1)}(g)A^{(i-1)}A^{(i-1)-1} \\
&= V^{(i)}\Pi^{(i-1)}(g).
\end{align*}
Since $\Pi^{(i-1)}(g)$ is a block-diagonal matrix, then blockwise equivariance follows.

\textbf{(b) } %
By part (a), we have
\begin{align*}
g^{(i)}(x) 
&= W^{(i)} f^{(i)}(x) \\
&= V^{(i)} A^{(i-1)} f^{(i)}(x) \\
&= V^{(i)} h^{(i)}(x).
\end{align*}
To establish the recursion for $h^{(i)}$, first observe that $A^{(i)}$ satisfies the recursion
\[ A^{(i)} = 
\lmat
I_{n_{i+1}} & -\frac{1}{2}W^{(i)} \\
0 & A^{(i-1)}
\rmat. \]
We thus have
\begin{align*}
h^{(i+1)}(x) 
&= A^{(i)} f^{(i+1)}(x) \\
&= 
\lmat
I_{n_{i+1}} & -\frac{1}{2}W^{(i)} \\
0 & A^{(i-1)}
\rmat \lmat
\relu(W^{(i)}f^{(i)}(x)+b^{(i)}) \\
f^{(i)}(x) 
\rmat \\
&= 
\lmat 
\relu(W^{(i)}f^{(i)}(x)+b^{(i)}) - \frac{1}{2}W^{(i)}f^{(i)}(x) \\
A^{(i-1)} f^{(i)}(x)
\rmat \\
&= 
\lmat
\relu(g^{(i)}(x)+b^{(i)}) - \frac{1}{2}g^{(i)}(x) \\
h^{(i)}(x)
\rmat,
\end{align*}
which completes the proof.
\end{proof}

\subsubsection{Implementation}
\label{app:forward}

Here we describe how Thm.~\ref{thm:V} translates into a concrete implementation. 
First, independently for every $i\in\{1,\ldots,d\}$ and $j\in\{1,\ldots,i\}$, we construct a basis set $\mathcal{B}^{(i)}_j$ for the space of all matrix solutions to the linear system
\[ \rho^{(i)}(g) X = X \pi^{(j-1)}(g)\forall g\in G. \]
This construction is done just once at the time of architecture initialization. 
We describe the construction below; 
for now, however, let us assume we have found such basis sets. 
Then we next constrain each latent weight matrix block $V^{(i)}_j$ to be a linear combination in $\mathcal{B}^{(i)}_j$. 
The linear coefficients are randomly initialized and are the trainable parameters of the $G$-DNN model. 
The forward pass is then a direct implementation of Thm.~\myref{thm:V}{b}, 
and its pseudocode is presented in Alg.~\ref{alg:forward}.

\begin{algorithm}
\caption{\label{alg:forward} %
Implementation of Thm.~\myref{thm:V}{b} for the forward pass of a $G$-DNN $f$. 
Here $g^{(i)}(x)$ and $h^{(i)}(x)$ should be regarded as variable names.
}
\begin{algorithmic}[1]
\Function{$f$}{$x$}
\State $h^{(1)}(x) \gets x$
\For{$i\in\{1,\ldots,d-1\}$}
	\State $g^{(i)}(x) \gets V^{(i)} h^{(i)}(x)$
	\State $h^{(i+1)}(x) \gets \lmat \relu(g^{(i)}(x)+b^{(i)}) - \frac{1}{2}g^{(i)}(x) \\ h^{(i)}(x) \rmat$
\EndFor
\State \Return $V^{(d)}g^{(d)}(x) + b^{(d)}$
\EndFunction
\end{algorithmic}
\end{algorithm}

All that remains is to describe the construction of the basis sets $\mathcal{B}^{(i)}_j$. 
We phrase this as the following general problem. 
Given a signed perm-rep $\rho$, an ordinary perm-rep $\pi$, and a generating set $G_0$ for the group $G$, 
we seek a basis set for the space of matrix solutions to the linear system
\[ \rho(g)X\pi(g)^{\top} = X\forall g\in G_0. \]
While we could proceed with standard methods of numerical linear algebra, there is a risk that numerical instabilities in the computation of the rank could lead to an incorrect number of basis matrices, which could break equivariance. 
We thus take a combinatorial approach to obtain an exact basis set as described next.

We construct a simple directed graph $\Gamma$ to encode the (signed) permutation of matrix elements under the mapping $X\rightarrow \rho(g)X\pi(g)^{\top}\forall g\in G_0$. 
The nodes of $\Gamma$ are the ordered pairs $(i, j)$ indexing the elements in $X$, 
and we draw an arc from $(i, j)$ to $(k, \ell)$ iff the former is sent to the latter under the above mapping for some $g\in G_0$; 
moreover, we assign the arc a value $z_{ij,k\ell}\in\{-1, 1\}$ to indicate whether a sign flip is incurred. 
With the graph $\Gamma$ in hand, a matrix $X$ is a solution to the linear equivariance condition iff all constraints $x_{ij} = z_{ij,k\ell}x_{k\ell}$ are satisfied. 
We are thus able to assign node values $x_{ij}$ independently across the connected components of $\Gamma$; 
based on this, we generate the desired basis set by processing each connected component $C$ as follows:
\begin{enumerate}[label={\textbf{Case~\arabic*.}}, align=left]
\item Suppose $C$ is $2$-colorable, meaning that each node $(i, j)\in C$ can be assigned a value $x_{ij}\in\{-1, 1\}$ such that the constraints $x_{ij} = z_{ij,k\ell}x_{k\ell}$ are satisfied within $C$. 
Then we select such an assignment (by greedy $2$-coloring), and we assign all nodes outside $C$ the value zero. 
These values together yield a basis matrix.
\item Suppose $C$ is not $2$-colorable. 
Then the only consistent assignment of values inside $C$ is all zeros, 
and no basis matrix is returned for this component.
\end{enumerate}
Observe that no two basis matrices share a nonzero value in the same position, 
and hence the matrices are indeed linearly independent.

\subsection{Admissible architectures}
\label{app:admissible}

\subsubsection{The \texorpdfstring{$\theta$}{theta} function}

The following proposition establishes the invariance and equivariance of the function $\theta$ with respect to certain conjugations.

\begin{proposition} \label{prop:theta_conj}
The function $\theta$ satisfies the following properties:
\begin{enumerate}[label={(\alph*)}]
\item \label{prop:theta_conj:a}
For every $A\in\operatorname{P}(|G/J|)$, define the ordinary perm-rep $\pi^A_J(g) = A\pi_J(g)A^{-1}$. 
Then $\theta(\cdot, \cdot, J)$ is invariant under the conjugation $\pi_J\mapsto\pi^A_J$.
\item \label{prop:theta_conj:b}
$\theta(H, K, gJg^{-1}) = \theta(H, K, J)\forall g\in G$.
\item \label{prop:theta_conj:c}
$\theta(gHg^{-1}, gKg^{-1}, J) = g\theta(H, K, J)g^{-1}$.
\end{enumerate}
\end{proposition}
\begin{proof}
\textbf{(a) } %
For brevity, let $\kappa = |H:K|-1$. 
Define the function $\theta^A$ in identical fashion to $\theta$, 
but replace $\pi_J$ with $\pi^A_J$. 
We wish to prove $\theta^A = \theta$. 
Since the map $\Gamma\mapsto P_{\Gamma}$ from finite orthogonal matrix groups to orthogonal projection operators is equivariant with respect to conjugation, then we have
\begin{align*}
\theta^A(H, K, J) 
&= \{g\in G: \pi^A_J(g)(P_{\pi_J^A(K)}-\kappa P_{\pi_J^A(H)}) = P_{\pi_J^A(K)}-\kappa P_{\pi_J^A(H)}\} \\
&= \{g\in G: A\pi_J(g)A^{-1}(P_{A\pi_J(K)A^{-1}}-\kappa P_{A\pi_J(H)A^{-1}}) = P_{A\pi_J(K)A^{-1}}-\kappa P_{A\pi_J(H)A^{-1}}\} \\
&= \{g\in G: A\pi_J(g)A^{-1}(AP_{\pi_J(K)}A^{-1}-\kappa AP_{\pi_J(H)}A^{-1}) = AP_{\pi_J(K)}A^{-1} \\
&\quad -\kappa AP_{\pi_J(H)}A^{-1}\} \\
&= \{g\in G: \pi_J(g)(P_{\pi_J(K)}-\kappa P_{\pi_J(H)}) = P_{\pi_J(K)}-\kappa P_{\pi_J(H)}\} \\
&= \theta(H, K, J).
\end{align*}

\textbf{(b) } %
The theory of ordinary perm-irreps and their correspondence to group action on cosets is well-understood~\citep{burnside1911theory, bouc2000burnside}, 
and it is known that the conjugation of $J$ in $\pi_J$ is equivalent to the conjugation of $\pi_J$ itself. 
The claim thus follows by part (a).

\textbf{(c) } %
Let $a\in G$. 
We will show $\theta(aHa^{-1}, aKa^{-1}, J) = \theta(H, K, J)$. 
Note $\kappa = |H:K|-1$ is invariant under the conjugation of $(H, K)$ by $a$. 
Also note $\pi_J(aHa^{-1}) = \pi_J(a)\pi_J(H)\pi_J(a)^{-1}$ and similar for $K$. 
Letting $A = \pi_J(a)$ and proceeding in analogy to part (a), we have
\begin{align*}
\theta(aHa^{-1}, aKa^{-1}, J) 
&= \{g\in G: \pi_J(g)(P_{\pi_J(aKa^{-1})}-\kappa P_{\pi_J(aHa^{-1})}) = P_{\pi_J(aKa^{-1})} \\
&\quad -\kappa P_{\pi_J(aHa^{-1})}\} \\
&= \{g\in G: \pi_J(g)(P_{\pi_J^A(K)}-\kappa P_{\pi_J^A(H)}) = P_{\pi_J^A(K)}-\kappa P_{\pi_J^A(H)}\} \\
&= \{g\in G: \pi_J(g)(AP_{\pi_J(K)}A^{-1}-\kappa AP_{\pi_J(H)}A^{-1}) = AP_{\pi_J(K)}A^{-1} \\
&\quad -\kappa AP_{\pi_J(H)}A^{-1}\} \\
&= \{g\in G: A^{-1}\pi_J(g)A(P_{\pi_J(K)}-\kappa P_{\pi_J(H)}) = P_{\pi_J(K)}-\kappa P_{\pi_J(H)}\} \\
&= \{g\in G: \pi_J(a^{-1}ga)(P_{\pi_J(K)}-\kappa P_{\pi_J(H)}) = P_{\pi_J(K)}-\kappa P_{\pi_J(H)}\}.
\end{align*}
By the change of variables $g\rightarrow aga^{-1}$, we have
\begin{align*}
\theta(aHa^{-1}, aKa^{-1}, J) 
&= \{aga^{-1}\in G: \pi_J(g)(P_{\pi_J(K)}-\kappa P_{\pi_J(H)}) = P_{\pi_J(K)}-\kappa P_{\pi_J(H)}\} \\
&= a\{g\in G: \pi_J(g)(P_{\pi_J(K)}-\kappa P_{\pi_J(H)}) = P_{\pi_J(K)}-\kappa P_{\pi_J(H)}\}a^{-1} \\
&= a\theta(H, K, J)a^{-1},
\end{align*}
establishing the claim.
\end{proof}

\begin{algorithm}
\caption{\label{alg:theta} %
Implementation of the function $\theta$ that exploits existing functions in GAP.
}
\begin{algorithmic}[1]
\Function{$\theta$}{$H$, $K$, $J$}
\Def{$\eta:K\setminus G/J\mapsto \operatorname{power}(G/J)$}
	\State $\eta(KxJ) = \{kxJ\in G/J: k(K\cap xJx^{-1})\in K/(K\cap xJx^{-1})\}$
\EndDef
\If{$|H:K|=2$}
	\State \Let $h\in H\setminus K$
	\State $S\gets \{KxJ\in K\setminus G/J: hx\in KxJ\}$
\Else
	\State $S\gets \emptyset$
\EndIf
\If{$S=\emptyset$}
	\State $T\gets \{\eta(KxJ): KxJ\in K\setminus G/J\}$
\Else
	\State $T\gets \{\eta(KxJ): KxJ\in(K\setminus G/J)\setminus S\}\cup \left\{\bigcup_{KxJ\in S} \eta(KxJ)\right\}$
\EndIf
\State \Return $\st_G(T)$
\EndFunction
\end{algorithmic}
\end{algorithm}

Algorithm~\ref{alg:theta} gives the pseudocode for an implementation of the function $\theta$ that can be accomplished in the GAP language~\citep{anon2021gap} for computational group theory. 
Although the definition of the $\theta$ function involves orthogonal projection operators, 
Alg.~\ref{alg:theta} completely circumvents these operators by taking a pure group-theoretic approach in terms of double cosets. 
The following proposition verifies that Alg.~\ref{alg:theta} is a correct implementation.

\begin{proposition} \label{prop:theta_alg}
Algorithm~\ref{alg:theta} correctly implements the function $\theta$.
\end{proposition}
\begin{proof}
Observe that the function $\eta$ in Alg.~\ref{alg:theta} sends each double coset $KxJ\in K\setminus G/J$ to the set of cosets in $G/J$ whose disjoint union is $KxJ$.

Let $H,K,J\leq G$ such that $K\leq H$ and $|H:K|\leq 2$. 
Let $w\in\ran(P_{\pi_J(K)}-(|H:K|-1)P_{\pi(H)})$. 
We regard $w$ as a function $w:G/J\mapsto\RR$. 
Since $w\in\ran(P_K)$, then $w$ is $K$-invariant in the sense that $w(kxJ) = w(xJ)$ for all $k\in K$ and $xJ\in G/J$. 
Thus, $w$ is constant over the set $\eta(KxJ)$ for every $KxJ\in K\setminus G/J$.

If $|H:K|=2$, then let $h\in H\setminus K$. 
Since $w\in\ran(P_{\pi_J(K)}-P_{\pi_J(H)})$, then $w(hKxJ) = -w(KxJ)$ for every $KxJ\in K\setminus G/J$. 
Thus, if $hKxJ = KxJ$, then $w(KxJ) = 0$. 
Since $K\trianglelefteq H$, then $hKxJ = KxJ$ is equivalent to $KhxJ = KxJ$, or just $hx\in KxJ$. 
We thus have the constraint
\[ w(KxJ)=0\forall KxJ\in K\setminus G/J\mid hx\in KxJ. \]

Now select $w$ such that it takes a different nonzero value over each $\eta(KxJ)$ for all $KxJ\in K\setminus G/J$ such that, if $|H:K|=2$, then $hx\notin KxJ$. 
Then
\begin{align*}
\theta(H, K, J)
&= \st_G(P_{\pi_J(K)}-(|H:K|-1)P_{\pi_J(H)}) \\
&= \st_G(w) \\
&= \{g\in G: w(gKxJ) = w(KxJ)\forall KxJ\in K\setminus G/J\}.
\end{align*}
This means $\theta(H, K, J)$ is exactly the subgroup of $G$ that leaves the level sets of $w$ invariant. 
Observe, however, that the sets in the collection $T$ in Alg.~\ref{alg:theta} are exactly the level sets of $w$, 
and hence $\theta(H, K, J) = \st_G(T)$.
\end{proof}

\subsubsection{The \texorpdfstring{$\phi$}{phi} function}

\begin{proof}[Proof of Prop.~\ref{thm:phi}]
The necessity of condition (2) is only because if $H^{(1)}_j=G$ for any $j$, then at least one row $w$ of $W^{(1)}$ satisfies $w\in\ran(P_G)$ by Thm.~4a of \citet{agrawal2022classification}. 
For $w$ to be nonzero, we thus require $P_G\neq 0$. 
We assume condition (2) to be satisfied for the remainder of the proof.

Before proving the claim itself, we first derive a closed-form expression for $\phi^{(i+1)}$. 
For notational convenience, for every $K\leq H\leq G$, $|H:K|\leq 2$, and linear rep $\tau:G\mapsto\operatorname{GL}(n, \RR)$, define
\[ \theta(H, K; \tau) = \{g\in G: \tau(g)(P_{\tau(K)}-(|H:K|-1)P_{\tau(H)}) = P_{\tau(K)}-(|H:K|-1)P_{\tau(H)}\}. \]
Thus, $\theta(H, K, J)$ can be equivalently written as $\theta(H, K; \pi_J)$. 
For two reps $\tau_1$ and $\tau_2$, observe that $P_{\tau_1\oplus\tau_2} = P_{\tau_1}\oplus P_{\tau_2}$ and hence
\[ \theta(H, K; \tau_1\oplus\tau_2) = \theta(H, K; \tau_1)\cap\theta(H, K; \tau_2). \]
This property extends to more than two reps in the obvious way. 
With this notation, and recalling the identity rep $\pi^{(0)}:G\mapsto G$, the function $\phi^{(i+1)}$ can be rewritten as
\begin{align*}
\phi^{(i+1)}(H, K) 
&= \st_G(P_K-(|H:K|-1)P_H)\cap \bigcap_{j=1}^i \bigcap_{r=1}^{r^{(j)}} \theta(H, K, H^{(j)}_r) \\
&= \theta(H, K; \pi^{(0)})\cap \bigcap_{j=1}^i \bigcap_{r=1}^{r^{(j)}} \theta(H, K; \pi^{(j)}_r) \\
&= \theta(H, K; \pi^{(0)})\cap \theta\left(H, K; \bigoplus_{j=1}^i \bigoplus_{r=1}^{r^{(j)}} \pi^{(j)}_r\right) \\
&= \theta\left(H, K; \pi^{(0)}\oplus\bigoplus_{j=1}^i \pi^{(j)}\right) \\
&= \theta\left(H, K; \bigoplus_{j=0}^i \pi^{(j)}\right) \\
&= \theta(H, K; \Pi^{(i)}).
\end{align*}
This is explicitly
\begin{align*}
\phi^{(i+1)}(H, K) 
&= \{g\in G: \Pi^{(i)}(g)(P_{\Pi^{(i)}(K)}-(|H:K|-1)P_{\Pi^{(i)}(H)}) = P_{\Pi^{(i)}(K)} \\
&\quad -(|H:K|-1)P_{\Pi^{(i)}(H)}\},
\end{align*}
and it is equivalent to
\begin{align*}
\phi^{(i+1)}(H, K) 
&= \{g\in G: \Pi^{(i)}(g)(P_{\Pi^{(i)}(K)}-(|\Pi^{(i)}(H):\Pi^{(i)}(K)|-1)P_{\Pi^{(i)}(H)}) = P_{\Pi^{(i)}(K)} \\
&\quad -(|\Pi^{(i)}(H):\Pi^{(i)}(K)|-1)P_{\Pi^{(i)}(H)}\}. \tag{*}
\end{align*}
To see that $|\Pi^{(i)}(H):\Pi^{(i)}(K)| = |H:K|$, by the First Isomorphism Theorem we have
\[ |\Pi^{(i)}(H):\Pi^{(i)}(K)| = \frac{|H|/|H\cap\ker(\Pi^{(i)})|}{|K|/|K\cap\ker(\Pi^{(i)})|}. \]
Since $\Pi^{(i)}$ includes in its direct sum decomposition the identity rep $\pi^{(0)}$, then $\ker(\Pi^{(i)})$ must be trivial, 
and so the above expression reduces to $|H|/|K| = |H:K|$.

We finally prove the claim. 
By definition, the $G$-DNN architecture $\{\rho^{(1)}, \ldots, \rho^{(d)}\}$ is admissible iff each row of $W^{(i)}$ is nonzero and no two rows of $[W^{(i)}\mid b^{(i)}]$ are parallel, for $i\in\{1,\ldots,d\}$. 
By Thm.~\myref{thm:V}{a}, since $W^{(i)} = V^{(i)}A^{(i-1)}$, then we obtain an equivalent definition if we replace $W^{(i)}$ with $V^{(i)}$. 
Let $V^{(i)}_{(j)}$ be the submatrix comprising the rows (and all columns) of $V^{(i)}$ that together transform by $\rho^{(i)}_j$:
\[ \rho^{(i)}_j(g)V^{(i)}_{(j)} = V^{(i)}_{(j)}\Pi^{(i-1)}(g)\forall g\in G. \]
(We include the parentheses in the subscript of $V^{(i)}_{(j)}$ to distinguish it from $V^{(i)}_j$ appearing in Thm.~\myref{thm:V}{a}). 
Then Thm.~4a of \citet{agrawal2022classification} implies we have an admissible architecture (specifically, that the rows of $V^{(i)}_{(j)}$ are nonzero and no two rows of the corresponding augmented weight matrix are parallel) iff
\[ \st_{\Pi^{(i-1)}(G)}(P_{\Pi^{(i-1)}(K^{(i)}_j)} - (|\Pi^{(i-1)}(H^{(i)}_j):\Pi^{(i-1)}(K^{(i)}_j)|-1) P_{\Pi^{(i-1)}(H^{(i)}_j)}) = \Pi^{(i-1)}(K^{(i)}_j), \]
or equivalently,
\[ (\Pi^{(i-1)})^{-1}[\st_{\Pi^{(i-1)}(G)}(P_{\Pi^{(i-1)}(K^{(i)}_j)} - (|\Pi^{(i-1)}(H^{(i)}_j):\Pi^{(i-1)}(K^{(i)}_j)|-1) P_{\Pi^{(i-1)}(H^{(i)}_j)}) = K^{(i)}_j \]
\begin{align*}
&\{g\in G: \Pi^{(i-1)}(g)(P_{\Pi^{(i-1)}(K^{(i)}_j)} - (|\Pi^{(i-1)}(H^{(i)}_j):\Pi^{(i-1)}(K^{(i)}_j)|-1) P_{\Pi^{(i-1)}(H^{(i)}_j)}) = P_{\Pi^{(i-1)}(K^{(i)}_j)} \\
&\quad - (|\Pi^{(i-1)}(H^{(i)}_j):\Pi^{(i-1)}(K^{(i)}_j)|-1) P_{\Pi^{(i-1)}(H^{(i)}_j)} = K^{(i)}_j.
\end{align*}
Recalling (*), we recognize the last equation as nothing but $\phi^{(i)}(H^{(i)}_j, K^{(i)}_j) = K^{(i)}_j$, there proving that condition (1) is (together with condition (2)) is equivalent to admissibility.
\end{proof}

\subsection{Additional remarks}
\label{app:remarks}

The following proposition establishes the compatibility of batchnorm with $G$-DNNs.

\begin{proposition} \label{prop:bn}
The addition of batchnorm immediately after any ReLU layer in a $G$-DNN preserves $G$-invariance of the network.
\end{proposition}
\begin{proof}
Suppose we apply batchnorm immediately after the $i$th ReLU layer of the $G$-DNN $f$, for some $i\in\{1,\ldots,d-1\}$. 
Then the activations $r^{(i)}(x) = \relu(W^{(i)}f^{(i)}(x)+b^{(i)})$ of the $i$th ReLU layer transform as
\[ r^{(i)}(x) \rightarrow \gamma \left(\frac{r^{(i)}(x)-\mu\vec{1}}{\sigma+\eps}\right) + \beta\vec{1}, \]
where $\mu\geq 0$, $\sigma\geq 0$, $\eps>0$, $\gamma$, and $\beta$ are all scalars. 
For each $i\in\{1,\ldots,d\}$, let $W^{(i)}_i$ and $V^{(i)}_i$ be the blocks comprising the first $n_i$ columns of $W^{(i)}$ and $V^{(i)}$ respectively; 
these represent the weights of the $i$th layer without the skip connections. 
Then the above affine transformation of $r^{(i)}(x)$ is equivalent to the transformation
\begin{align*}
W^{(i+1)}_{i+1} &\rightarrow C W^{(i+1)}_{i+1} \\
b^{(i+1)} &\rightarrow b^{(i+1)} + D W^{(i+1)}_{i+1} \vec{1},
\end{align*}
where
\begin{align*}
C &= \frac{\gamma}{\sigma+\eps} \\
D &= \beta - \frac{\gamma\mu}{\sigma+\eps}.
\end{align*}
By Thm.~\myref{thm:V}{a}, $W^{(i+1)} = V^{(i+1)}A^{(i)}$. 
Since $A^{(i)}$ is upper block-triangular with the top left block $I_{n_{i+1}}$, 
then $W^{(i+1)}_{i+1} = V^{(i+1)}_{i+1}$. 
By the above transformation of $W^{(i+1)}_{i+1}$ under batchnorm, $V^{(i+1)}_{i+1}$ also transforms only by the scalar factor $C$, 
and hence it remains equivariant as in Thm.~\ref{thm:V}{a}.

All that remains is to show the bias $b^{(i+1)}$ satisfies the sufficient condition in Lemma~\ref{lemma:psi} even after the batchnorm transformation, and this will establish $G$-invariance of the network with batchnorm.

\textbf{Case 1: } %
Suppose $\rho^{(i+1)}$ is type 1 irreducible. 
Then Lemma~\ref{lemma:psi} implies $b^{(i+1)}$ is parallel to $\vec{1}$. 
Under batchnorm, $b^{(i+1)}$ transforms by the addition of $DV^{(i+1)}_{i+1}\vec{1}$ and thus remains parallel to $\vec{1}$. 
The claim follows by Lemma~\ref{lemma:psi}.

\textbf{Case 2: } %
Suppose $\rho^{(i+1)}$ is type 2 irreducible. 
Then Lemma~\ref{lemma:psi} implies $b^{(i+1)} = 0$. 
Under batchnorm, the bias thus transforms to $0+DV^{(i+1)}_{i+1}\vec{1}$. 
It turns out, however, that $V^{(i+1)}_{i+1}\vec{1}=0$; 
to see this, by Thm.~\myref{thm:V}{a}, we have
\[ \rho^{(i+1)}(g)V^{(i+1)}_{i+1} = V^{(i+1)}_{i+1} \pi^{(i)}(g)\forall g\in G. \]
Since $i\geq 1$ so that $\pi^{(i)}$ is type 1, then without loss of generality, by selecting an appropriate basis, we assume $\pi^{(i)}$ is an ordinary perm-irrep. 
Averaging both sides over all $g\in G$, we obtain
\[ P_{\rho^{(i+1)}} V^{(i+1)}_{i+1} = V^{(i+1)}_{i+1} P_{\pi^{(i)}}. \]
Since $\rho^{(i+1)}$ is type 2, then its only fixed point is the zero vector, and hence the orthogonal projection operator $P_{\rho^{(i+1)}}$ is itself zero. 
Moreover, since $\pi^{(i)}$ is an ordinary perm-irrep and since $\vec{1}$ is fixed under all permutations, then it is fixed under the orthogonal projection operator $P_{\pi^{(i)}}$ as well-- 
hence $V^{(i+1)}_{i+1}\vec{1} = 0$.

\textbf{Case 3: } %
Suppose $\rho^{(i+1)}$ is reducible. 
Then decompose it into type1 and type 2 irreps and apply Cases 1-2 separately to each irrep.
\end{proof}

The extension of Prop.~\ref{prop:bn} to multiple channels is trivial. 
Batchnorm is typically applied independently to each channel. 
Thus, if the ReLU activations $r^{(i)}(x)$ had $c$ channels (which could be achieved by having $c$ copies of every irrep in $\rho^{(i)}$), 
then the variables $\mu$, $\sigma$, $\gamma$, and $\beta$ would be $c$-dimensional vectors. 
The proof would then proceed by selecting a single arbitrary channel.

\section{Examples}
\label{app:examples}

\subsection{Concatenated ReLU}

\begin{proof}[Proof of Ex.~\ref{ex:crelu}] 
For every $i\in\{1,\ldots,d\}$, the function
\[ x\rightarrow U^{(i)}\crelu(U^{(i-1)}\cdots\crelu(U^{(1)}x)\cdots) \]
is $G$-equivariant with its output transforming by $\rho^{(i)}$. 
We thus identify it with the function $g^{(i)}$ appearing in Thm.~\myref{thm:V}{b}. 
We now have the recursion
\[ g^{(i+1)}(x) = U^{(i+1)}\crelu( g^{(i)}(x) )\forall i\in\{1,\ldots,d-1\}. \]
By definition of $\crelu(\cdot)$ and the block structure of $U^{(i+1)}$, we have
\begin{align*}
g^{(i+1)}(x) 
&= \lmat U^{(i+1)}_1 & U^{(i+1)}_2 \rmat \relu\left( \lmat g^{(i)}(x) \\ -g^{(i)}(x) \rmat \right) \\
&= U^{(i+1)}_1 \relu( g^{(i)}(x) ) + U^{(i+1)}_2 \relu( -g^{(i)}(x) ) \\
&= U^{(i+1)}_1 \relu( g^{(i)}(x) ) + U^{(i+1)}_2 \relu( g^{(i)}(x) ) - U^{(i+1)}_2 g^{(i)}(x) \\
&= (U^{(i+1)}_1+U^{(i+1)}_2) \relu( g^{(i)}(x) ) - U^{(i+1)}_2 g^{(i)}(x) \\
&= (U^{(i+1)}_1+U^{(i+1)}_2) \relu( g^{(i)}(x) ) - \frac{1}{2}(U^{(i+1)}_1+U^{(i+1)}_2) g^{(i)}(x) + \frac{1}{2}(U^{(i+1)}_1-U^{(i+1)}_2) g^{(i)}(x) \\
&= \lmat U^{(i+1)}_1+U^{(i+1)}_2 & \frac{1}{2}(U^{(i+1)}_1-U^{(i+1)}_2) \rmat \lmat \relu(g^{(i)}(x)) - \frac{1}{2}g^{(i)}(x) \\ g^{(i)}(x) \rmat.
\end{align*}
Since $g^{(i)}(x) = V^{(i)} h^{(i)}(x)$ in Thm.~\myref{thm:V}{b} and since all bias vectors in this example are zero, then we can write the above as
\[ g^{(i+1)}(x) = \lmat U^{(i+1)}_1+U^{(i+1)}_2 & \frac{1}{2}(U^{(i+1)}_1-U^{(i+1)}_2)V^{(i)} \rmat \lmat \relu(g^{(i)}(x)+b^{(i)}) - \frac{1}{2}g^{(i)}(x) \\ h^{(i)}(x) \rmat. \]
Comparing this to the definitions of $g^{(i+1)}$ and $h^{(i+1)}$ appearing in Thm.~\myref{thm:V}{b}, we obtain the  expression for $V^{(i+1)}$ as claimed.
\end{proof}

\subsection{Binary multiplication}

\begin{proof}[Proof of Ex.~\ref{ex:prod}] 
\textbf{(a) } %
For $x_1,x_2\in \{-1, 1\}$, it can be verified by hand that
\begin{align*}
x_1x_2 
&= \lmat 1 & -1\rmat \relu\left( \lmat 1 & 1 \\ 1 & -1 \rmat \lmat x_1 \\ x_2 \rmat \right) - x_2 \\
&= \lmat 1 & -1 & 0 & -1 \rmat \lmat \relu\left( \lmat 1 & 1 \\ 1 & -1 \rmat \lmat x_1 \\ x_2 \rmat \right) \\ \lmat x_1 \\ x_2 \rmat \rmat. (*)
\end{align*}
We can extend this to the product of more than two elements as follows: 
For $i\in\{1,\ldots,d\}$, define the block-diagonal matrices
\begin{align*}
A^{(i)} &= I_{m/2^i} \lmat 1 & -1 \rmat \\
B^{(i)} &= I_{m/2^i} \lmat 0 & -1 \rmat \\
C^{(i)} &= I_{m/2^i} \lmat 1 & 1 \\ 1 & -1 \rmat,
\end{align*}
and define the functions $p^{(i)}:\mathcal{X}\mapsto \{-1, 1\}^{m/2^i}$ by
\begin{align*}
p^{(1)}(x) &= A^{(1)} x \\
p^{(i)}(x) &= \lmat A^{(i)} & B^{(i)} \rmat \lmat \relu( C^{(i)} p^{(i-1)}(x)) \\ p^{(i-1)}(x) \rmat\forall i\in\{2,\ldots,d\}. (**)
\end{align*}
The function $p^{(1)}$ maps each pair of elements $(x_{2i-1}, x_{2i})$ in the input $x$ to $x_{2i-1}-x_{2i}\in\{-1, 1\}$. 
Then $p^{(i)}$ 
(1) partitions $p^{(1)}(x)$ into blocks, each of two elements; 
(2) takes the product of each pair of elements using (*); 
(3) iterates this procedure $i-1$ times. 
The output $p^{(d)}(x)$ is then the product of the elements in $p^{(1)}(x)$.

Now for $i\in\{1,\ldots,d-1\}$, the function $g^{(i)}$ (as in Thm.~\myref{thm:V}{b}) as given by $g^{(i)}(x) = C^{(i+1)} p^{(i)}(x)$. 
By (**), we thus have
\begin{align*}
g^{(i)}(x) 
&= \lmat C^{(i+1)}A^{(i)} & C^{(i+1)}B^{(i)} \rmat \lmat \relu( g^{(i-1)}(x)) \\ p^{(i-1)}(x) \rmat \\
&= \lmat C^{(i+1)}A^{(i)} & C^{(i+1)}B^{(i)} \rmat \lmat \relu( g^{(i-1)}(x)) - \frac{1}{2}g^{(i-1)}(x) + \frac{1}{2}g^{(i-1)}(x) \\ p^{(i-1)}(x) \rmat \\
&= \lmat C^{(i+1)}A^{(i)} & C^{(i+1)}B^{(i)} \rmat \lmat \relu( g^{(i-1)}(x)) - \frac{1}{2}g^{(i-1)}(x) \\ p^{(i-1)}(x) \rmat \\
&\quad + \lmat C^{(i+1)}A^{(i)} & C^{(i+1)}B^{(i)} \rmat \lmat \frac{1}{2}g^{(i-1)}(x) \\ 0 \rmat \\
&= \lmat C^{(i+1)}A^{(i)} & C^{(i+1)}B^{(i)} \rmat \lmat \relu( g^{(i-1)}(x)) - \frac{1}{2}g^{(i-1)}(x) \\ p^{(i-1)}(x) \rmat 
+ \frac{1}{2}C^{(i+1)}A^{(i)} g^{(i-1)}(x) \\
&= \lmat C^{(i+1)}A^{(i)} & C^{(i+1)}B^{(i)} \rmat \lmat \relu( g^{(i-1)}(x)) - \frac{1}{2}g^{(i-1)}(x) \\ p^{(i-1)}(x) \rmat \\
&\quad + \frac{1}{2}C^{(i+1)}A^{(i)} C^{(i)} p^{(i-1)}(x) \\
&= \lmat C^{(i+1)}A^{(i)} & C^{(i+1)}B^{(i)} + \frac{1}{2}C^{(i+1)}A^{(i)} C^{(i)} \rmat \lmat \relu( g^{(i-1)}(x)) - \frac{1}{2}g^{(i-1)}(x) \\ p^{(i-1)}(x) \rmat \\
&= \lmat C^{(i+1)}A^{(i)} & 0 \rmat \lmat \relu( g^{(i-1)}(x)) - \frac{1}{2}g^{(i-1)}(x) \\ p^{(i-1)}(x) \rmat \\
&= \lmat C^{(i+1)}A^{(i)} & 0 \rmat \lmat \relu( g^{(i-1)}(x)) - \frac{1}{2}g^{(i-1)}(x) \\ h^{(i-1)}(x) \rmat.
\end{align*}
Comparing this to the equations in Thm.~\myref{thm:V}{b}, we establish the claimed expression for $V^{(i)}_j$ for $i\in\{1,\ldots,d-1\}$.

For the case $i=d$, we simply observe that the last weight matrix must be the outer weight vector in (*). 
The $G$-invariance of the constructed DNN is clear; 
the action of any $g\in G$ on an input $x\in\mathcal{X}$ corresponds to an even number of sign flips in $p^{(1)}(x)$, 
which leaves the parity of the product $p^{(d)}(x)$ invariant. 
Layerwise $G$-equivariance is established next.

\textbf{(b) } %
By part (a), The first weight matrix is
\[ C^{(2)}A^{(1)} = I_{m/4}\otimes \lmat 1 & -1 & 1 & -1 \\ 1 & -1 & -1 & 1 \rmat. \]
The $j$th pair of rows in this weight matrix is the $j$th channel, 
whose first row is $v_j$ and which transforms by $\rho_{H^{(1)}_j K^{(1)}_j}$. 
The expressions for $H^{(1)}_j$ and $K^{(1)}_j$ are thus established by definition.

Now for $i\in\{1,\ldots,d\}$, part (a) implies
\[ \rho_{H^{(i+1)}_j K^{(i+1)}_j}(g) \lmat 1 & -1 & 1 & -1 \\ 1 & -1 & -1 & 1 \rmat = \lmat 1 & -1 & 1 & -1 \\ 1 & -1 & -1 & 1 \rmat \lmat \pi_{H^{(i)}_{2j-1}}(g) & 0 \\ 0 & \pi_{H^{(i)}_{2j}}(g) \rmat\forall g\in G, \]
where each $\pi_{H^{(i)}_k}$ is the ordinary perm-rep part of $\rho_{H^{(i)}_k K^{(i)}_k}$. 
By definition, $K^{(i+1)}_j$ is the subgroup of all $g\in G$ such that
\begin{align*}
\rho_{H^{(i+1)}_jK^{(i+1)}_j}(g)e_1 &= e_1 \\
\frac{1}{2} \rho_{H^{(i+1)}_j K^{(i+1)}_j}(g) \lmat 1 & -1 & 1 & -1 \\ 1 & -1 & -1 & 1 \rmat \lmat 1 \\ 0 \\ 1 \\ 0 \rmat &= e_1 \\
\frac{1}{2} \lmat 1 & -1 & 1 & -1 \\ 1 & -1 & -1 & 1 \rmat \lmat \pi_{H^{(i)}_{2j-1}}(g) & 0 \\ 0 & \pi_{H^{(i)}_{2j}}(g) \rmat\forall g\in G \lmat 1 \\ 0 \\ 1 \\ 0 \rmat &= e_1.
\end{align*}
Since ordinary perm-reps cannot flip signs, then the last equation is equivalent to
\[ \pi_{H^{(i)}_{2j-1}} e_1 = e_1 \mbox{ and } \pi_{H^{(i)}_{2j}} e_1 = e_1. \]
We thus have
\begin{align*}
K^{(i+1)}_j 
&= \{g\in G: \pi_{H^{(i)}_{2j-1}}e_1 = e_1\} \cap \{g\in G: \pi_{H^{(i)}_{2j}}e_1 = e_1\} \\
&= H^{(i)}_{2j-1} \cap H^{(i)}_{2j}.
\end{align*}

Similarly, by definition we have
\[ H^{(i+1)}_j = \{g\in G: \rho_{H^{(i+1)}_jK^{(i+1)}_j}(g)e_1 = \pm e_1\}. \]
Proceeding analogously as above, we find that $g\in G$ is contained in $H^{(i+1)}_j$ iff 
\[ (\pi_{H^{(i)}_{2j-1}}(g)e_1 = e_1 \mbox{ and } \pi_{H^{(i)}_{2j}}(g)e_1 = e_1) 
\mbox{ or } 
(\pi_{H^{(i)}_{2j-1}}(g)e_1 = e_2 \mbox{ and } \pi_{H^{(i)}_{2j}}(g)e_1 = e_2). \]
The first term in the disjunction corresponds to $H^{(i)}_{2j-1}\cap H^{(i)}_{2j}$, 
and the second term in the disjunction corresponds to the intersection of the complements, as claimed.

In the case of the $(d-1)$st layer, the product output $p^{(d-1)}(x)$ is 2D and thus can only transform by $\pm 1$. 
The rep $\rho^{(d-)}$ thus decomposes into two copies of a scalar rep. 
Finally, that the final rep $\rho^{(d)}$ is trivial follows from the $G$-invariance of the network, 
thereby completing the proof.
\end{proof}


\vskip 0.2in
\bibliography{references}

\end{document}